\documentclass[preprint]{elsarticle}
\makeatletter
\def\ps@pprintTitle{%
\let\@oddhead\@empty
\let\@evenhead\@empty
\def\@oddfoot{\scriptsize{\it Preprint accepted in Neurocomputing}\phantom{\scriptsize 5 nov 2019} \hfill {\normalsize \thepage} \hfill  \phantom{\it\scriptsize Preprint accepted in Neurocomputing}5 nov 2019}%
\let\@evenfoot\@oddfoot}
\makeatother

\usepackage{lineno}
\modulolinenumbers[5]
\usepackage{hyperref}       
\usepackage{url}            
\usepackage{booktabs}       
\usepackage{amsfonts}       
\usepackage{nicefrac}       
\usepackage{microtype}      
\usepackage{amsmath,amssymb}
\usepackage{dsfont}
\usepackage{bbm}
\usepackage{graphicx}
\usepackage{color}
\usepackage{bbm}
\usepackage{multirow}
\usepackage{url}
\usepackage{balance}
\usepackage{multicol}
\usepackage{mathtools}
\usepackage{arydshln}
\usepackage{color}
\usepackage{xspace}
\usepackage{mathtools}
\usepackage{upgreek}
\usepackage[table]{xcolor} 
\usepackage{booktabs}
\usepackage{url}
\usepackage{marvosym}
\usepackage{enumerate}
\usepackage{subfig}
\usepackage{fancyhdr}
\usepackage{fancybox}
\usepackage{todonotes}

\newtheorem{corollary}{Corollary}
\newtheorem{lemma}{Lemma}
\newtheorem{theorem}{Theorem}
\newtheorem{definition}{Definition}

\newcommand{\xbf}{\ensuremath{\mathbf{x}}}
\newcommand{\xb}{\xbf}
\newcommand{\sbf}{\xbf}

\newcommand{\tbf}{\xbf}

\newcommand{\exbfs}{(\sbf,y)\sim\PS}
\newcommand{\exbft}{(\tbf,y)\sim\PT}
\newcommand{\xbfs}{\sbf\sim\DS}
\newcommand{\exbf}{(\xbf,y)}
\newcommand{\zerobf}{{\mathbf 0}}

\newcommand{\mt}{{m_t}}
\newcommand{\ms}{{m_s}}
\newcommand{\m}{\textsc{m}}
\newcommand{\wbf}{\mathbf{w}}
\newcommand{\wb}{\wbf}
\newcommand{\ab}{{\pmb{\alpha}}}
\newcommand{\alphab}{\ab}

\newcommand{\Hcal}{\ensuremath{\mathcal{H}}}
\newcommand{\hdh}{\Hcal\!\Delta\!\Hcal}
\newcommand{\X}{{\bf X}}
\newcommand{\Y}{Y}
\newcommand{\XY}{\X\times\Y}
\newcommand{\R}{\mathbb{R}}
\newcommand{\Rbb}{\R}

\newcommand{\TminusS}{{\PT{\setminus}\PS}}

\newcommand{\posterior}{\rho}
\newcommand{\prior}{\pi}
\newcommand{\Q}{\posterior}
\newcommand{\PT}{{\cal T}}
\newcommand{\PS}{{\cal S}}
\newcommand{\DS}{{\cal S}_{\X}}
\newcommand{\DT}{{\cal T}_{\X}}
\newcommand{\D}{{\cal D}_{\X}}
\renewcommand{\P}{{\cal D}}

\newcommand{\Ncal}{{\mathcal N}}

\newcommand{\BQ}{B_\posterior}
\newcommand{\GQ}{G_\posterior}

\newcommand{\Risk}{{\rm R}}
\newcommand{\RP}{\Risk_\P}
\newcommand{\RD}{\Risk_{\D}}
\newcommand{\RPT}{\Risk_{\PT}}
\newcommand{\RPS}{\Risk_{\PS}}
\newcommand{\RS}{\widehat{\Risk}_S}
\newcommand{\RT}{\widehat{\Risk}_T}
\newcommand{\RDS}{\Risk_{\DS}}
\newcommand{\RDT}{\Risk_{\DT}}
\newcommand{\loss}{\mathcal{L}}
\newcommand{\zoloss}{\loss_{\zo}}

\renewcommand{\d}{{\rm d}}
\newcommand{\dD}{\d_{\D}}
\newcommand{\dDT}{{\d}_{\DT}}
\newcommand{\dDS}{{\d}_{\DS}}
\newcommand{\dT}{\widehat{\d}_T}
\newcommand{\dS}{\widehat{\d}_S}
\newcommand{\e}{{\rm e}}
\newcommand{\eP}{\e_{\P}}
\newcommand{\ep}{\eP}
\newcommand{\ePT}{\e_{\PT}}
\newcommand{\ePS}{\e_{\PS}}
\newcommand{\eS}{\widehat{\e}_{S}}

\newcommand{\Phirisk}{\Phi_{\Risk}}
\newcommand{\Phic}{{\widetilde{\Phi}}_{\Risk}}
\newcommand{\Phidis}{\Phi_{\d}}
\newcommand{\Phierr}{\Phi_{\e}}

\newcommand{\disc}{\operatorname{disc}}
\newcommand{\disoperator}{\operatorname{dis}}
\newcommand{\des}{\disoperator_\posterior}
\newcommand{\dis}{\des} 
\newcommand{\desw}{\disoperator_{\posterior_\wb}}
\newcommand{\desST}{\widehat{\disoperator}_{\posterior}(S,T)}
\newcommand{\deswST}{\widehat{\disoperator}_{\posterior_\wb}(S,T)}

\newcommand{\betainf}{\beta_{\infty}}
\newcommand{\betaq}{\beta_{q}}
\newcommand{\bq}[1][q]{\beta_{#1}(\PT\|\PS)}
\newcommand{\binf}{\bq[\infty]}
\newcommand{\bqx}[1][q]{\beta_{#1}(\DT\|\DS)}

\newcommand{\PBDA}{{\sc pbda}\xspace}
\newcommand{\DALC}{\algo}
\newcommand{\algo}{{\sc dalc}\xspace}
\newcommand{\SVM}{{\sc svm}\xspace}

\newcommand{\CODA}{{\sc coda}\xspace}
\newcommand{\DASVM}{{\sc dasvm}\xspace}
\newcommand{\PBGD}{{\sc pbgd3}\xspace}

\newcommand{\eqdef}{=} 
\DeclareMathOperator*{\Esp}{\mathbf{E}}

\DeclareMathOperator*{\argmin}{\mathrm{argmin}}

\newcommand{\Ibf}{{\mathbf I}}
\newcommand{\I}{\ensuremath{\mathbf{I}}}
\renewcommand{\I}{\mathrm{I}}
\newcommand{\Erf}{\ensuremath{\textrm{Erf}}}
\newcommand{\sign}{\operatorname{sign}}
\newcommand{\sgn}{\sign}
\newcommand{\scriptsupport}{\mbox{\scriptsize\sc supp}}
\newcommand{\support}{\mbox{\small\sc supp}}
\newcommand{\KL}{{\rm KL}}
\newcommand{\kl}{{\rm kl}}

\newcommand{\zo}{0\textrm{-}\!1}
\newcommand{\pii}{{\boldsymbol{\uppi}}}
\newcommand{\ets}{\eta_\TminusS}
\newcommand{\Bcal}{\mathcal{B}}

\newcommand{\LP}{\left(}
\newcommand{\RPP}{\right)}

\newcommand{\ie}{{\em i.e.\/}}
\newcommand{\eg}{{\em e.g.\/}}

\newcommand{\eqdots}{\coloneqq}

\newcommand{\degree}{\ensuremath{^\circ}}

\usepackage{numcompress}

\begin{document}

\begin{frontmatter}

\title{PAC-Bayes and Domain Adaptation}
\author[inria,ulaval]{Pascal Germain}
\author[ujm]{Amaury Habrard}
\author[ulaval]{Francois Laviolette}
\author[ujm]{Emilie Morvant}

\address[ujm]{Univ Lyon, UJM-Saint-Etienne, CNRS, Institut d Optique Graduate School,
\\Laboratoire Hubert Curien UMR 5516, F-42023, Saint-Etienne, France}
\address[ulaval]{D\'epartement d'Informatique et de G\'enie Logiciel, Universit\'e Laval, Qu\'ebec, Canada}
\address[inria]{Inria Lille - Nord Europe, Modal Project-Team, 59650 Villeneuve d'Ascq, France}

\begin{abstract}
We provide two main contributions in PAC-Bayesian theory for domain adaptation where the objective is to learn, from a source distribution, a well-performing majority vote on a different, but related, target distribution. 
Firstly, we propose an improvement of the previous approach we proposed in~\cite{pbda}, which relies on a novel distribution pseudodistance based on a disagreement averaging, allowing us to derive a new tighter domain adaptation bound for the target risk.
While this bound stands in the spirit of common domain adaptation works, we derive a second bound (introduced in~\cite{dalc}) that brings a new perspective on domain adaptation by deriving an upper bound on the target risk where the distributions' divergence---expressed as a ratio---controls the trade-off between a source error measure and the target voters' disagreement. 
We discuss and compare both results, from which we obtain PAC-Bayesian generalization bounds. Furthermore, from the PAC-Bayesian specialization to linear classifiers, we infer two learning algorithms, and we evaluate them on real data.
\end{abstract}


\end{frontmatter}

\section{Introduction}

As human beings, we learn from what we saw before. 
Think about our education process: When a student attends to a new course, the knowledge he has acquired from previous courses helps him to understand the current one.
However, traditional machine learning approaches assume that the learning and test data are drawn from the same probability distribution.
This assumption may be too strong for a lot of real-world tasks, in particular those where we desire to reuse a model from one task to another one.
For instance, a spam filtering system suitable for one user can be poorly adapted to another who receives significantly different emails.
In other words, the learning data associated with one or several users could be unrepresentative of the test data coming from another one.
This enhances the need to design methods for adapting a classifier from learning (source) data to test (target) data.
One solution to tackle this issue is to consider the {\it domain adaptation} framework,\footnote{The reader can refer to the surveys proposed in \citep{JiangSurvey08,Quionero-Candela:2009,Margolis2011,wang2018deep,wouter2019,BookDATheory} (domain adaptation is often associated with \emph{transfer learning}~\citep{Pan-TL-TKDE09}).} which arises when the distribution generating the target data (the {\it target domain}) differs from the one generating the source data (the {\it source domain}).
Note that, it is well known that domain adaptation is a hard and challenging task even under strong assumptions~\citep{David-AISTAT10,BenDavid12,Ben-DavidU14}. 

\subsection{Approaches to Address Domain Adaptation}

Many approaches exist in the literature to address domain adaptation, often with the same underlying idea: If we are able to apply a transformation in order to ``move closer'' the distributions, then we can learn a model with the available labels.
This process can be performed by reweighting the importance of labeled data~\citep{HuangSGBS-nips06,Sugiyama-NIPS07,CortesMM10,cortes-15}. 
This is one of the most popular methods when one wants to deal with the covariate-shift issue \citep[\emph{e.g.,}][]{HuangSGBS-nips06,sugiyama2008direct}, where source and target domains diverge only in their marginals, {\it i.e.}, when they share the same labeling function.
Another technique is to exploit self-labeling procedures, where the objective is to transfer the source labels to the target unlabeled points \citep[\eg,][]{BruzzoneM10S,habrard2013iterative,PVMinCq}.
A third solution is to learn a new common representation space from the unlabeled part of source and target data.
Then, a standard supervised learning algorithm can be run on the source labeled \citep[\eg,][]{glorot2011domain,Chen12,courty2016optimal,CourtyFTR17,LiLH0S19}, or can be learned simultaneously with the new representation; the latter method being increasingly used in state-of-the-art deep neural networks learning strategies \citep[\eg,][]{ganin-16,DingF18,ShuBNE18,LiLHZS19,SebagHSSWA19}.
A slightly different approach, known as hypothesis transfer learning, aims at directly transferring the learned source model to the target domain~\cite{kuzborskij2013stability,ThesisKuzborskij2018theory}.

The work presented in this paper stands into a fifth 
popular class of approaches, which has been especially explored to derive generalization bounds for domain adaptation. 
This kind of approach relies on the control of a measure of divergence/distance between the source distribution and target distribution \citep[\eg][]{BenDavid-NIPS06,BenDavid-MLJ2010,li2007bayesian,Zhang12,morvant12,CortesM14,redko2017theoretical}.
Such a distance usually depends on the set $\Hcal$ of hypotheses considered by the learning algorithm. 
The intuition is that one must look for a set $\Hcal$ that minimizes the distance between the distributions while preserving good performances on the source data; if the distributions are close under this measure, then generalization ability may be ``easier'' to quantify.
In fact, defining such a measure to quantify how much the domains are related is a major issue in domain adaptation.
For example, for binary classification with the $\zo$-loss function, \citeauthor{BenDavid-MLJ2010}~\cite{BenDavid-MLJ2010,BenDavid-NIPS06} 
have considered the \emph{\mbox{{\small $\hdh$}-divergence}} between the source and target marginal distributions. 
This quantity depends on the maximal disagreement between two classifiers, and  allowed them to deduce a domain adaptation generalization bound based on the \mbox{VC-dimension} theory. 
The \emph{discrepancy distance} proposed by \citet{Mansour-COLT09} generalizes this divergence to real-valued functions and more general losses, and is used to obtain a generalization bound based on the Rademacher complexity.
In this context, \citet{CortesM11,CortesM14} have specialized the minimization of the discrepancy to regression with kernels.
In these situations, domain adaptation can be viewed as a multiple trade-off between the complexity of the hypothesis class $\Hcal$, the adaptation ability of~$\Hcal$ according to the divergence between the marginals, and the empirical source risk.
Moreover, other measures have been exploited under different assumptions, such as the R{\'e}nyi divergence suitable for importance weighting \citep{MansourMR09}, or the measure proposed by \citet{Zhang12} which takes into account the source and target true labeling, or the Bayesian \emph{divergence prior}  \citep{li2007bayesian} which favors classifiers closer to the best source model, or the Wassertein distance~\cite{redko2017theoretical} that justifies the usefulness of optimal transport strategy in domain adaptation~\cite{courty2016optimal,CourtyFTR17}.
However, a majority of methods prefer to perform a two-step approach: {\it (i)}~First construct a suitable representation by minimizing the divergence, then {\it (ii)}~learn a model on the source domain in the new representation space.

\subsection{A PAC-Bayesian Standpoint}

Given the multitude of concurrent approaches for domain adaptation, and the nonexistence of a predominant one, we believe that the problem still needs to be studied from different perspectives for a global comprehension to emerge.
We aim to contribute to this study from a PAC-Bayesian standpoint.
One particularity of the {\it  PAC-Bayesian theory} (first set out by \citet{Mcallester99a}) is that it focuses on algorithms that output a {\it posterior distribution} $\posterior$ over a classifier set~$\Hcal$ ({\it i.e.}, a \mbox{$\posterior$-average} over $\Hcal$) rather than just a single predictor $h\in\Hcal$ (as in~\cite{BenDavid-NIPS06}, and other works cited above).
More specifically, we tackle the \emph{unsupervised domain adaptation} setting for binary classification, where no target labels are provided to the learner.
We propose two domain adaptation analyses, both
introduced separately in previous conference papers \citep{pbda,dalc}. We refine these  results, and provide in-depth comparison, full proofs and technical details.
Our analyses highlight different angles that one can adopt when studying domain adaptation.

Our first approach follows the philosophy of the seminal work of \citeauthor{BenDavid-MLJ2010}~\cite{BenDavid-MLJ2010,BenDavid-NIPS06} and~\citet{MansourMR09}: The risk of the target model is upper-bounded jointly by the model's risk on the source distribution, a divergence between the marginal distributions, and a non-estimable term\footnote{More precisely, this term can only be estimated in the presence of labeled data from both the source and the target domains.} related to the ability to adapt in the current space. 
To obtain such a result, we define a pseudometric which is ideal for the PAC-Bayesian setting by evaluating the domains' divergence according to the \mbox{$\posterior$-average} disagreement of the classifiers over the domains. 
Additionally, we prove that this domains' divergence is always lower than the popular \mbox{{\small $\hdh$}-divergence}, and is easily estimable from samples.
Note that, based on this disagreement measure, we derived in a previous work \citep{pbda} a first PAC-Bayesian domain adaptation bound expressed as a \mbox{$\posterior$-averaging}.
We provide here a new version of this result, that does not change the underlying philosophy supported by the previous bound, but clearly 
improves the theoretical result: The domain adaptation bound is now tighter and easier to interpret.

Our second analysis (introduced in~\cite{dalc}) consists in a target risk bound that brings an original way to think about domain adaptation problems.
Concretely, the risk of the target model is still upper-bounded by three terms, but they differ in the information they capture.
The first term is estimable from unlabeled data and relies on the disagreement of the classifiers only on the target domain.
The second term depends on the expected accuracy of the classifiers on the source domain. 
Interestingly, this latter is weighted by a divergence between the source and the target domains that enables controlling the relationship between domains.
The third term estimates the ``volume'' of the target domain living apart from the source one,\footnote{Here we do not focus on learning a new representation to help the adaptation: We directly aim at adapting in the current representation space.} which has to be small for ensuring adaptation.

Thanks to these results, we derive PAC-Bayesian generalization bounds for our two domain adaptation bounds.
Then, in contrast to the majority of methods that perform a two-step procedure, we design two algorithms tailored to linear classifiers, called \PBDA and \DALC, which jointly minimize the multiple trade-offs implied by the bounds.
On the one hand, \PBDA is inspired by our first analysis for which the first two quantities being, as usual in the PAC-Bayesian approach, the complexity of the \mbox{$\posterior$-weighted} majority vote measured by a Kullback-Leibler divergence and the empirical risk measured by the \mbox{$\posterior$-average} errors on the source sample. 
The third quantity corresponds to our domains' divergence and assesses the capacity of the posterior distribution to distinguish some structural difference between the source and target samples.
On the other hand, \DALC is inspired by our second analysis from which we deduce that a good adaptation strategy consists in finding a \mbox{$\posterior$-weighted} majority vote leading to a suitable trade-off---controlled by the domains' divergence---between the first two terms (and the usual Kullback-Leibler divergence): 
Minimizing the first one corresponds to look for classifiers that disagree on the target domain, and minimizing the second one to seek accurate classifiers on the source.

\subsection{Paper Structure and Novelties}

This paper aims at unifying contributions of our previous papers~\cite{pbda,dalc}. 
We have also added the following contributions.
\begin{itemize}
    \item Subsection~\ref{sec:pbgd} presents the explicit derivation of the algorithm {\small PBGD3}~\citep{germain2009pac} (see also the mathematical details of \ref{appendix:RSGw}). An original illustration of the algorithm optimized trade-off is also given (Subsection~\ref{section:pbgd_illustration} and Figure~\ref{fig:phiclassic}).
    \item Theorem~\ref{theo:pbda} introduces an improved version of the original PAC-Bayesian domain adaptation bound \cite{pbda}.  As discussed in Subsection~\ref{sec:pbda_quantities}, this new theorem provides tighter generalization guarantees and is easier to interpret. Moreover, the bound is not degenerated when the source and target distributions are the same or close, which was an undesirable behavior of the previous result.
    \item Section~\ref{sec:pb_da_bounds} presents comprehensive and reorganized proofs of the main PAC-Bayesian results. In this new version, both Theorem~\ref{theo:pbda} (improved version of \cite{pbda}) and Theorem~\ref{theo:dalc} (from \cite{dalc}) build on a common result, given by Corollary~\ref{theo:catoni_new}, instead of being proven independently.
    \item Section~\ref{sec:dapbgd} gives the extended mathematical details leading to the two learning algorithms (\PBDA~\cite{pbda} and \DALC~\cite{dalc}), including the equation of their kernelized version (Subsection~\ref{section:kernelDA}). Moreover, an original toy experiment illustrates the particularities of the two algorithms in regards of each other (Subsection~\ref{section:dalc_illustration} and Figure~\ref{fig:phida}).
\end{itemize}


The rest of the paper is structured as follows.
Section~\ref{sec:notations} deals with two seminal works on domain adaptation. 
The PAC-Bayesian framework is then recalled in Section~\ref{sec:pacbayes}, along with the details of {\small PBGD3} algorithm \citep{germain2009pac}.
Our main contribution, which consists in two domain adaptation bounds suitable for PAC-Bayesian learning, is presented in Section~\ref{sec:two_da_bounds}, the associated generalization bounds are derived in Section~\ref{sec:pb_da_bounds}.
Then, we design our new algorithms for PAC-Bayesian domain adaptation in Section~\ref{sec:dapbgd}, that we empirically evaluate in Section~\ref{sec:expe}.
We conclude in Section~\ref{sec:conclu}.

\section{Domain Adaptation Related Works}
\label{sec:notations}
\label{sec:da}

In this section, we review the two seminal works in domain adaptation that are based on a divergence measure between the domains~\cite{BenDavid-MLJ2010,BenDavid-NIPS06,Mansour-COLT09}.

\subsection{Notations and Setting}
We consider domain adaptation for binary classification\footnote{Our domain adaptation analysis is tailored for binary classification, and does not directly extend to multi-class and regression problems.} tasks where $\X  \subseteq   \mathbb{R}^d$ is the input space of dimension $d$, and \mbox{$\Y  =  \{-1, +1\}$} is the output/label set.
The {\it source domain} $\PS$ and the {\it target domain} $\PT$ are two different distributions (unknown and fixed) over $\X\times\Y$, $\DS$ and $\DT$ being the respective marginal distributions over $\X$. 
We tackle the challenging task where we have no target labels, known as \emph{unsupervised domain adaptation}. 
A learning algorithm is then provided with a {\it labeled source sample} $S = \{(\sbf_i,y_i)\}_{i=1}^{\ms}$ consisting of $\ms$ examples drawn {\it i.i.d.}\footnote{{\it i.i.d.} stands for {\it independent and identically distributed}.} from $\PS$, and an {\it unlabeled target sample} $T = \{\tbf_j\}_{j=1}^{\mt}$ consisting of $\mt$ examples drawn {\it i.i.d.} from $\DT$.
We denote the distribution $\P$ of a \mbox{$m$-sample} by $(\P)^m$.
We suppose that $\Hcal$ is a set of hypothesis functions for $\X$ to $Y$.
The {\it expected source error} and the {\it expected target error} of $h\in\Hcal$ over $\PS$, respectively~$\PT$, are the probability that~$h$ errs on the entire distribution $\PS$, respectively~$\PT$,
$$
\RPS(h) \ \eqdef \, \Esp_{(\sbf,y) \sim \PS} \zoloss\big( h(\sbf), y \big)\,,\quad\mbox{and}\quad \RPT(h) \ \eqdef \, \Esp_{(\tbf,y) \sim \PT} \zoloss\big( h(\tbf), y \big)\,,
$$
where $\zoloss(a,b)  \eqdef  \I[a {\ne} b]$ is the $\zo$-loss function which returns $1$ if $a {\ne}   b$ and $0$ otherwise.  The {\it empirical source error} $\RS(h)$ of $h$ on the learning source sample~$S$~is
$$
\RS(h)  \ \eqdef \ \frac{1}{\ms}\sum_{(\sbf,y) \in S} \zoloss\big( h(\sbf), y \big)\,.
$$
The main objective in domain adaptation is then to learn---without target labels---a classifier $h\in\Hcal$ leading to the lowest expected target error $\RPT(h)$. \\

Given two classifiers $(h',h)\in\Hcal^2$, we also introduce the notion of {\it expected source disagreement} $\RDS(h,h')$ and the {\it expected target disagreement} $\RDT(h,h')$, which measure the probability that $h$ and $h'$ do not agree on the respective marginal distributions, and are defined by
$$
\RDS(h,h') \ \eqdef\ \, \Esp_{\mathclap{\sbf\sim \DS}}\  \zoloss\big( h(\sbf), h'(\sbf) \big)\,\ 
\mbox{ and }\ \RDT(h,h') \ \eqdef\ \, \Esp_{\mathclap{\tbf\sim \DT}} \ \zoloss\big( h(\tbf), h'(\tbf) \big)\,.
$$
The {\it empirical source disagreement} $\RS(h, h')$ on $S$ and the {\it empirical target disagreements} $\RT(h, h')$ on $T$ are
$$
\RS(h,h')  \eqdef \frac{1}{\ms}\sum_{\sbf\in S}\! \zoloss\big( h(\sbf), h'(\sbf) \big)\,\mbox{ and }\RT(h,h') \eqdef  \frac{1}{\mt}\!\sum_{\tbf\in T} \zoloss\big( h(\tbf), h'(\tbf) \big).
$$
Note that, depending on the context, $S$ denotes either the source labeled sample $\{ (\sbf_i,y_i) \}_{i=1}^{\ms} $ or its unlabeled part $\{\sbf_i\}_{i=1}^{\ms} $.
We can remark that the expected error $R_{\P}(h)$ on a distribution $\P$ can be viewed as a shortcut notation for the expected disagreement between a hypothesis $h$ and a labeling function $f_{\P}:\X\to\Y$ that assigns the true label to an example description  with respect to $\P$. 
We have
\begin{align*}
 \Risk_{\P}(h) \ &=\ \Risk_{\D}(h, f_{\P}) \\ 
&=\  \Esp_{\xb\sim \D} \zoloss\big( h(\xb), f_{\P}(\xb) \big)\,. 
\end{align*}

\subsection{Necessity of a Domains' Divergence}
\label{sec:necessity_dist}

The domain adaptation objective is to find a low-error target hypothesis, even if the target labels are not available. 
Even under strong assumptions, this task can be impossible to solve \citep{David-AISTAT10,BenDavid12,Ben-DavidU14}. 
However, for deriving generalization ability in a domain adaptation situation (with the help of a domain adaptation bound), it is critical to make use of a divergence between the source and the target domains: The more similar the domains, the easier the adaptation appears.
Some previous works have proposed different quantities to estimate how a domain is close to another one \citep{BenDavid-NIPS06,li2007bayesian,Mansour-COLT09,MansourMR09,BenDavid-MLJ2010,Zhang12}.
Concretely, two domains $\PS$ and $\PT$ differ if their marginals $\DS$ and $\DT$ are different, or if the source labeling function differs from the target one, or if both happen.
This suggests taking into account two divergences: One between $\DS$ and $\DT$, and one between the labeling.
If we have some target labels, we can combine the two distances as done by \citet{Zhang12}.
Otherwise, we preferably consider two separate measures, since it is impossible to estimate the best target hypothesis in such a situation. 
Usually, we suppose that the source labeling function is somehow related to the target one, then we look for a representation where the marginals $\DS$ and $\DT$ appear closer without losing performances on the source domain. 

\subsection{Domain Adaptation Bounds for Binary Classification}

We now review the first two seminal works which propose domain adaptation bounds based on a divergence between the two domains.

First, under the assumption that there exists a hypothesis in $\Hcal$ that performs well on both the source and the target domain, \citeauthor{BenDavid-MLJ2010}~\cite{BenDavid-NIPS06,BenDavid-MLJ2010} have provided the following domain adaptation bound.
\begin{theorem}[\citet{BenDavid-MLJ2010,BenDavid-NIPS06}]
\label{theo:BenDavid}
Let ${\cal H}$ be a (symmetric\footnote{In a symmetric hypothesis space $\Hcal$, for every $h\in\Hcal$, its inverse $-h$ is also in $\Hcal$.}) hypothesis class. We have
\begin{equation}
\label{eq:da}
\forall h\in {\cal H},\  \RPT(h)\, \leq\, \RPS(h) + \tfrac{1}{2}d_{\hdh}(\DS,\DT) + \mu_{h^*}\, , 
\end{equation}
where  $$\tfrac{1}{2}d_{\hdh}(\DS,\DT) \ \eqdef \, \displaystyle  \sup_{\substack{(h,h')\in\mathcal{H}^2}}  \left|\RDT(h, h') - \RDS(h, h')\right|$$ 
is the \mbox{{\small $\hdh$}-distance} between  marginals $\DS$ and $\DT$, and
\mbox{$\mu_{h^*}  {\eqdef} \RPS(h^{ *})+\RPT(h^{ *})$}
is the error of the best hypothesis overall 
$h^{*}  {\eqdef}  \argmin_{h\in{\cal H}}  \big( \RPS (h) + \RPT (h) \big)\,.$ 
\end{theorem}
This bound relies on three terms. 
The first term $\RPS(h)$ is the classical source domain expected error. 
The second term $\tfrac12 d_{\Hcal \Delta \Hcal}(\DS,\DT)$ depends on $\Hcal$ and corresponds to the maximum deviation between the source and target disagreement between two hypotheses of $\Hcal$.
In other words, it quantifies how hypothesis from $\Hcal$ can ``detect'' differences between these marginals: The lower this measure is for a given $\Hcal$, the better are the generalization guarantees. 
The last term $\mu_{h^*}=\RPS(h^{ *})+\RPT(h^{ *})$  is related to the best hypothesis $h^{*}\in\Hcal$ over the domains and acts as a quality measure of $\Hcal$ in terms of labeling information. 
If $h^{*}$ does not have a good performance on both the source and the target domain, then there is no way one can adapt from this source to this target. 
Hence, as pointed out by the authors, Equation~\eqref{eq:da} expresses a multiple trade-off between the accuracy of some particular hypothesis $h$, the complexity of~$\Hcal$ (quantified in \cite{BenDavid-MLJ2010} with the usual \mbox{VC-bound} theory), and the ``incapacity'' of hypotheses of $\Hcal$ to detect difference between the source and the target domain.

Second, \citet{Mansour-COLT09} have extended the \mbox{{\small $\hdh$}-distance} to the discrepancy divergence for regression and any symmetric loss $\loss$ fulfilling the triangle inequality.
Given \mbox{$\loss : [-1,+1]^2  \to  \R^+$} such a loss, the discrepancy $\disc_{\loss}(\DS,\DT)$ between $\DS $ and~$\DT$~is  
$$\displaystyle \disc_{\loss}(\DS,\DT)  \eqdef    \sup_{\substack{(h,h')\in\Hcal^2}}  \Big| \Esp_{\tbf\sim\DT}     \loss(h(\tbf),h'(\tbf)) -     \Esp_{\sbf\sim\DS}      \loss(h(\sbf),h'(\sbf))\Big|\,.$$
Note that with the $\zo$-loss in binary classification, we have
$$\tfrac{1}{2}d_{\Hcal \Delta \Hcal}(\DS,\DT)   \,=\,  \disc_{\zoloss}(\DS,\DT)\,.$$
Even if these two divergences may coincide, the following domain adaptation bound of \citet{Mansour-COLT09} differs from Theorem~\ref{theo:BenDavid}.
\begin{theorem}[\citet{Mansour-COLT09}]
\label{theo:Mansour}
Let ${\cal H}$ be a (symmetric) hypothesis class. We have
\begin{align}
\label{eq:dabounddisc} \forall h\in {\cal H},\ \RPT(h) - \RPT(h_T^*) \ \leq \   \RDS(h_S^*,h)  +\disc_{\zoloss}(\DS,\DT) + \nu_{(h_S^*,h_T^*)}&\,,
\end{align}
with $\nu_{(h_\PS^*,h_\PT^*)}\eqdef\RDS(h_\PS^*,h_\PT^*)$
the disagreement between the ideal hypothesis on the target and source domains:
$h_\PT^* \eqdef \argmin_{h\in\Hcal}  \RPT(h),$ and $h_\PS^*\eqdef \argmin_{h\in\Hcal}\RPS(h).$
\end{theorem}

Equation~\eqref{eq:dabounddisc} can be tighter than Equation \eqref{eq:da}\footnote{Equation~\eqref{eq:da} can lead to an error term three times higher than Equation~\eqref{eq:dabounddisc} in some cases (more details in~\cite{Mansour-COLT09}).} since it bounds the difference between the target error of a classifier and the one of the optimal $h_\PT^*$. 
Based on Theorem~\ref{theo:Mansour} and a Rademacher complexity analysis, \citet{Mansour-COLT09} provide a generalization bound on the target risk, that expresses a trade-off between the disagreement (between $h$ and the best source hypothesis $h_\PS^*$), the complexity of  $\Hcal$, and---again---the ``incapacity'' of hypotheses to detect differences between the domains.

To conclude, the domain adaptation bounds of Theorems~\ref{theo:BenDavid} and~\ref{theo:Mansour} suggest that if the divergence between the domains is low, a low-error classifier over the source domain might perform well on the target one.
These divergences compute the \emph{worst-case} of the disagreement between a pair of hypotheses. 
We propose in Section~\ref{sec:two_da_bounds} two \emph{average case} approaches by making use of the essence of the PAC-Bayesian theory, which is known to offer tight generalization bounds~\citep{Mcallester99a,germain2009pac,Parrado-Hernandez12}.
Our first approach (see Section~\ref{sec:first}) stands in the philosophy of these seminal works, and the second one (see Section~\ref{sec:second}) brings a different and novel point of view by taking advantages of the PAC-Bayesian framework we recall in the next section.

\section{PAC-Bayesian Theory in Supervised Learning}
\label{sec:pacbayes}

Let us now review the classical supervised binary classification framework called the PAC-Bayesian theory, first introduced by \citet{Mcallester99a}.
This theory succeeds to provide tight generalization guarantees---without relying on any validation set---on weighted majority votes, {\it i.e.}, for ensemble methods~\citep{dietterich2000ensemble,re2012ensemble} where several classifiers (or voters) are assigned a specific weight. 
Throughout this section, we adopt an algorithm design perspective. 
Indeed, the PAC-Bayesian analysis of domain adaptation provided in the forthcoming sections is oriented by the motivation of creating new adaptive algorithms.

\subsection{Notations and Setting}

Traditionally, PAC-Bayesian theory considers weighted majority votes over a set $\Hcal$ of binary hypothesis, often called voters.
Let $\P$ be a fixed yet unknown distribution over $\XY$, and~$S$ be a learning set where each example is drawn {\it i.i.d.}\ from $\P$.
Then, given a \emph{prior distribution}~$\prior$ over $\Hcal$ (independent from the learning set $S$), the ``PAC-Bayesian'' learner aims at finding a \emph{posterior distribution}~$\posterior$ over $\Hcal$ leading to a \emph{\mbox{$\posterior$-weighted} majority vote}  $B_\posterior$ (also called the Bayes classifier) with good generalization guarantees and defined by
$$B_\posterior(\xbf) \ \eqdef \ \sign\Big[\Esp_{h\sim \posterior} h(\xbf)\Big]\,.$$
However, minimizing the risk of $\BQ$, defined as 
$$
\RP(B_\posterior)\ \eqdef\  \Esp_{(\xbf,y) \sim \P} \zoloss\big( \BQ(\xbf), y \big)\,,
$$
 is known to be \mbox{NP-hard}. 
To tackle this issue, the PAC-Bayesian approach deals with the risk of the stochastic \emph{Gibbs classifier} $G_\posterior$ associated with $\posterior$ and closely related to $B_\posterior$. 
In order to predict the label of an example~$\xbf\in\X$, the Gibbs classifier first draws a hypothesis $h$ from $\Hcal$ according to $\posterior$, then returns $h(\xbf)$ as label. 
Then, the error of the Gibbs classifier on a domain $\P$ corresponds to the expectation of the errors over~$\posterior$:
\begin{align}
\label{eq:RGQ}
\RP(G_\posterior) \ \eqdef \ \Esp_{h\sim\posterior} \RP(h)\,.
\end{align}
In this setting, if $B_\posterior$ misclassifies $\xbf$, then at least half of the classifiers (under $ \posterior$) errs on $\xbf$. 
Hence, we have $$ \RP(B_\posterior) \ \leq  \  2\,\RP(G_\posterior)\,.$$ 
Another result on the relation between $\RP(B_\posterior)$ and $\RP(G_\posterior)$ is the \mbox{$C$-bound} of \citet{Lacasse07} expressed as
\begin{align}
\label{eq:C-bound}
\RP(B_\posterior) \ \leq\ 1-\frac{ \big(1-2\,\RP(G_\posterior)\big)^2}{1-2\,\dD(\Q)}\,,
\end{align}
where $\dD(\Q)$ corresponds to the {\it expected disagreement} of the classifiers over $\posterior$:
\begin{align}
\label{eq:RGQGQ}
\dD(\Q) \ \eqdef\, \Esp_{(h,h')\sim \posterior^2} \RD(h,h')\,.
\end{align}
Equation~\eqref{eq:C-bound} suggests that for a fixed numerator, {\it i.e.}, a fixed risk of the Gibbs classifier, the best \mbox{$\posterior$-weighted} majority vote is the one associated with the lowest denominator, {\it i.e.}, with the greatest disagreement between its voters 
(for further analysis, see \cite{graal-neverending}).

We now introduce the notion of {\it expected joint error} of a pair of classifiers $(h,h')\in\Hcal^2$ drawn according to the distribution $\posterior$, defined as
\begin{equation}
\label{eq:eP}
\eP(\Q) \  \eqdef\,   \Esp_{(h,h') \sim\posterior^2}\Esp_{(\xbf,y) \sim \P} \zoloss\big( h(\xb), y \big) \times \zoloss\big( h'(\xb), y \big)\,.
\end{equation}
From the definitions of the expected disagreement and the joint error, \citet{Lacasse07} (see also \cite{graal-neverending}) observed that in a binary classification context, given a domain~$\P$ on $\XY$ and a distribution $\posterior$ on~$\Hcal$, we can decompose the Gibbs risk as
\begin{equation} \label{eq:rde}
\RP(G_\posterior) \ = \ \frac{1}{2} \,\dD(\Q) + \eP(\Q)\,.
\end{equation}
Indeed, since $Y=\{-1,+1\}$, we have
\begin{align*}
2\, \RP(G_\posterior)
&=
\Esp_{(h,h') \sim\posterior^2}\Esp_{(\xbf,y) \sim \P} 
\Big[ \zoloss\big( h(\xb), y \big) + \zoloss\big(h'(\xb), y \big) \Big] \\
&=
\Esp_{(h,h') \sim\posterior^2}\Esp_{(\xbf,y) \sim \P} 
\Big[ 
\zoloss\big( h(\xb), h'(\xb) \big) + 2{\times}\zoloss\big(h(\xb), y \big)\, \zoloss\big(h'(\xb), y \big) \Big]\\[1.5mm]
&=
\dD(\Q) + 2\times \eP(\Q)\,.
\end{align*}
Lastly, PAC-Bayesian theory allows one to bound the expected error $\RP(G_\posterior)$ in terms of two major quantities: The empirical error $$\RS(G_\posterior)  \  =\  \Esp_{h\sim\posterior} \RS(h)$$ estimated on a sample~$S\sim(\P)^m$,  and the Kullback-Leibler divergence
\begin{equation*}
\KL(\posterior\|\prior)  \,\eqdef\,  \Esp_{h\sim \posterior}  \ln \frac{\posterior(h)}{\prior(h)}\,.
\end{equation*}
In the next section, we introduce a PAC-Bayesian theorem proposed by \mbox{\citet{catoni2007pac}}.\footnote{Two other common forms of the PAC-Bayesian theorem are the one of \citet{Mcallester99a} and the one of \citet{Seeger02,Langford05}. We refer the reader to our research report \citep{pbda_long} for a larger variety of PAC-Bayesian theorems in a domain adaptation context.}

\subsection{A Usual PAC-Bayesian Theorem}
\label{sec:threepb}

Usual PAC-Bayesian theorems suggest that, in order to minimize the expected risk, a learning algorithm should perform a trade-off between the empirical risk minimization $\RS(G_{\posterior})$ and \mbox{KL-divergence} minimization $\KL(\posterior\,\|\,\prior)$ (roughly speaking the complexity term).
The nature of this trade-off can be explicitly controlled in Theorem~\ref{thm:pacbayescatoni} below.
This PAC-Bayesian result, first proposed by \citet{catoni2007pac}, is defined with a hyperparameter (here named $\omega$).
It appears to be a natural tool to design PAC-Bayesian algorithms.
We present this result in the simplified form suggested by~\citet{germain09b}.    
\begin{theorem}[\citet{catoni2007pac}] 
\label{thm:pacbayescatoni}
For any domain $\P$ over  $\X   \times   Y$, for  any set of hypotheses~$\Hcal$,  any prior distribution $\prior$ over $\Hcal$, any $\delta \in (0,1]$, and any real number $\omega>0$,  with a probability at least $1 - \delta$ over the random choice of $S \sim  (\P)^{m} $, for every posterior distribution $\posterior$ on $\Hcal$, we have
\begin{equation} \label{eq:catoniclassic}
\RP(G_{\posterior}) \ \leq\ \frac{\omega}{1 - e^{-\omega}}    \left[\RS(G_{\posterior})  +  \frac{\KL(\posterior\|\prior) + \ln  \frac{1}{\delta}}{m \times \omega}\right].
\end{equation}
\end{theorem}
Similarly to \citet{mcallester-keshet-11}, we could choose to restrict $\omega\in(0,2)$ to obtain a slightly looser but simpler bound.
Using $e^{-\omega}\leq1-\omega-\frac12\omega^2$ to upper-bound on the right-hand side of Equation~\eqref{eq:catoniclassic}, we obtain
\begin{equation}\label{eq:mckeshet}
\RP(G_{\posterior}) \ \leq\ 
\frac{1}{1{-}\frac12 \omega} \left[ \RS(G_{\posterior})   + \frac{\KL(\posterior\|\prior) + \ln \frac{1}{\delta}}{ m\times \omega}\right].
\end{equation}
The bound of Theorem~\ref{thm:pacbayescatoni}---in both forms of Equations~\eqref{eq:catoniclassic} and~\eqref{eq:mckeshet}---has two appealing characteristics. 
First, choosing $\omega = 1/\sqrt m$, the bound becomes consistent: It converges to $1 \times \big[\,\RS(G_{\posterior})  +  0\,\big]$ as $m$~grows.
Second, as described in Section~\ref{sec:pbgd}, its minimization is closely related to the minimization problem associated with the Support Vector Machine (\SVM) algorithm when $\posterior$ is an isotropic Gaussian over the space of linear classifiers~\citep{germain2009pac}. 
Hence, the value $\omega$ allows us to control the trade-off between the empirical risk $\RS(G_{\posterior})$ and the ``complexity term'' $\tfrac{1}{m}\,\KL(\posterior\|\prior)$.

\subsection{Supervised PAC-Bayesian Learning of Linear Classifiers} 
\label{sec:pbgd}

Let us consider $\Hcal$ as a set of linear classifiers in a \mbox{$d$-dimensional space}. 
Each $h_{\wbf'}\in \Hcal$ is defined by a weight vector ${\wbf'}\in\mathbb{R}^d$:
$$h_{\wbf'}(\xb)\ \eqdef \ \sgn\LP\wbf'\cdot\xb\RPP,$$ 
where $\,\cdot\,$ denotes the dot product.

By restricting the prior and the posterior distributions over $\Hcal$ to be Gaussian distributions, \citet{Langford02} have specialized the PAC-Bayesian theory in order to bound the expected risk of any linear classifier $h_\wb\in\Hcal$.
More precisely, given a prior $\prior_{\mathbf{0}}$ and a posterior $\posterior_\wb$ defined as spherical Gaussians with identity covariance matrix respectively centered on vectors $\mathbf{0}$ and~$\wb$, for any  $h_{\wbf'}\in\Hcal$, we have
\begin{align*}
\prior_\mathbf{0}(h_{\wbf'})\, &\eqdef\, \LP \frac{1}{\sqrt{2\pii}} \RPP^{ d}
  \exp\left({-\frac{1}{2}\|{\wbf'}\|^2}\right),\\
\mbox{ and } \quad 
    \posterior_\wb(h_{\wbf'})  \, &\eqdef\,   \LP \frac{1}{\sqrt{2\pii}} \RPP^{ d}
  \exp\left({-\frac{1}{2}\|{\wbf'}-\wb\|^2}\right) .
 \end{align*}
 An interesting property of these distributions---also seen as multivariate normal distributions, $\prior_\zerobf=\Ncal(\zerobf, \Ibf)$ and $\posterior_\wb= \Ncal(\wbf, \Ibf)$---is that the prediction of the \mbox{$\posterior_\wb$-weighted} majority vote $B_{\posterior_\wb}$ coincides with the one of the linear classifier $h_\wb$. 
Indeed, we have
 \begin{align*}
\forall\, \xb\in \X,\ \forall\,\wb\in \Hcal, \quad 
h_\wb(\xb) \ &= \ B_{\posterior_\wb}(\xb) \\
 \ &= \ \sign\left[ \Esp_{h_{\wb'} \sim \posterior_\wb} h_{\wb'}(\xb) \right].
 \end{align*} 
 Moreover, the expected risk of the Gibbs classifier $G_{\posterior_\wb}$ on a domain $\P$ is then given by\footnote{The calculations leading to Equation~\eqref{eq:RSGw} can be found in \citet{Langford05}. For sake of completeness, we provide a slightly different derivation in~\ref{appendix:RSGw}.}
\begin{equation} 
\label{eq:RSGw}
\RP(G_{\posterior_\wb}) \   =  \Esp_{(\xb,y)\sim P_S} \Phirisk\left(  y\, \frac{\wb \cdot \xb}{\|\xb\|}  \right),
\end{equation}
where
\begin{align}
\label{eq:eq:gibbs_risk_linear}
\Phirisk(x)  \ \eqdef  \ \tfrac{1}{2}  \left[1 - \Erf\left( \tfrac{x}{\sqrt{2}} \right) \right],
\end{align}
with $\Erf$ is the Gauss error function defined as
\begin{align}
\label{eq:erf}
\Erf\,(x)\ \eqdef\ \frac{2}{\sqrt{\pii}}\ \int_{0}^{x} \exp\left(-t^2\right) \text{d}t\,.
\end{align}
Here, $\Phirisk(x)$ can be seen as a \emph{smooth} surrogate of the \zo-loss function $\I\left[x\leq 0\right]$ relying on $y\, \frac{\wb \cdot \xbf}{\|\xbf\|}$.
This function $\Phirisk$ is sometimes called the \emph{probit--loss} \citep[\eg,][]{mcallester-keshet-11}.
It is worth noting that $\|\wb\|$ plays an important role on the value of $\Risk_\P (G_{\posterior_\wb})$, but not on $\Risk_\P (h_\wb)$. 
Indeed, $\Risk_\P (G_{\posterior_\wb})$ tends to $\Risk_\P (h_\wb)$ as $\|\wb\|$ grows, which can provide \emph{very tight bounds} (see the empirical analyses of \cite{AmbroladzePS06,germain2009pac}).
Finally, the \mbox{KL-divergence} between $\posterior_\wb$ and $\prior_\mathbf{0}$ becomes simply 
 \begin{align*}
\KL\big(\posterior_\wb \| \prior_\mathbf{0}\big) 
 \ &= \ \KL\big( \Ncal(\wbf, \Ibf) \,\|\, \Ncal(\zerobf, \Ibf) \big)\\
&= \ \tfrac{1}{2}\| \wb \|^2\,, 
\end{align*}
and turns out to be a measure of \emph{complexity} of the learned classifier.

\subsubsection{Objective Function and Gradient}
\label{sec:pbgd_objective}

Based on the specialization of the PAC-Bayesian theory to linear classifiers, \citet{germain2009pac} suggested minimizing a PAC-Bayesian bound on $\RP(G_{\posterior_\wb})$. 
For sake of completeness, we provide here more mathematical details than in the original conference paper~\citep{germain2009pac}. 
In forthcoming Section~\ref{sec:pbda}, we will extend this supervised learning algorithm to the domain adaptation setting.

Given a sample $S = \{(\xbf_i, y_i)\}_{i=1}^{m}$ and a hyperparameter $\Omega>0$, the learning algorithm performs gradient descent in order to find an optimal weight vector $\wb$ that minimizes
\begin{eqnarray}
\label{eq:prob_pbgd_primal}
F(\wb) 
& = & 
\Omega\, m\, \RS (G_{\posterior_\wb})  + \KL(\posterior_\wb \| \prior_\mathbf{0}) \nonumber \\[1mm]
& = &
\Omega\,  \displaystyle\sum_{i=1}^{m} \Phirisk\left(  y \frac{\wb \cdot \xbf_i}{\|\xbf_i\|}  \right)  
 + \frac{1}{2}\|\wb\|^2\,.
\end{eqnarray}
It turns out that the optimal vector $\wb$ corresponds to the distribution $\posterior_\wb$ minimizing the value of the bound on $\RP(G_{\posterior_\wb})$ given by Theorem~\ref{thm:pacbayescatoni}, with the parameter $\omega$ of the theorem being the hyperparameter $\Omega$ of the learning algorithm. 
It is important to point out that PAC-Bayesian theorems bound simultaneously $\RP(G_{\posterior_\wb})$ \emph{for every $\posterior_\wb$ on $\Hcal$}. 
Therefore, one can ``freely'' explore the domain of objective function $F$ to choose a posterior distribution~$\posterior_\wb$ that gives, thanks to Theorem~\ref{thm:pacbayescatoni}, a bound valid with probability $1-\delta$.

The minimization of Equation~\eqref{eq:prob_pbgd_primal} by gradient descent corresponds to the learning algorithm called {\small PBGD3} of \citet{germain2009pac}.
The gradient of $F(\wb)$ is given the vector~$\nabla F (\wb)$:
\begin{align*}
\nabla F(\wb) \ = \ \Omega\, \sum_{i=1}^m 
\Phirisk' \LP y_i \frac{\wb\cdot\xb_i}{\|\xb_i\|} \RPP  \frac{y_i\,\xb_i}{\|\xb_i\|}
+ \wb\,,
\end{align*}
where $\Phirisk'(x) = -\tfrac{1}{\sqrt{2\pii}} \exp\left(-\tfrac{1}{2} x^2\right)$ is the derivative of $\Phirisk$ at point $x$.

\medskip
Similarly to {\small SVM}, the learning algorithm {\small PBGD3} realizes a trade-off between the empirical risk---expressed by the loss  $\Phirisk$---and the complexity of the learned linear classifier---expressed by the regularizer $\|\wb\|^2$.  
This similarity increases when we use a kernel function, as described next.

\subsubsection{Using a Kernel Function}
\label{section:pbgd_kernel}

The kernel trick allows substituting inner products by a kernel function $k:\R^d \times \R^d\rightarrow\R$ in Equation~\eqref{eq:prob_pbgd_primal}. 
If $k$ is a Mercer kernel, it implicitly represents a function $\phi:X\rightarrow\R^{d'}$ that maps an example of $\X$ into an arbitrary \mbox{$d'$-dimensional} space, such that 
$$\forall(\xbf,\xbf')\in \X^2,\quad k(\xb,\xb') \ =\ \phi(\xb)\cdot\phi(\xb')\,.$$ 
Then, a dual weight vector $\ab = (\alpha_1,\alpha_2, \ldots,\alpha_{m})\in\R^m$ encodes the linear classifier $\wb \in \R^{d'}$ as a linear combination of examples of $S$:
\begin{equation*}
\wb \ =\ 
\sum_{i=1}^m\alpha_i\, \phi(\xb_i)\,, 
\quad\mbox{ and thus } \quad
h_\wb(\xb) \ =\ 
\sgn\left[
\sum_{i=1}^m \alpha_i k(\xb_i, \xb)
\right].
\end{equation*}

By the representer theorem~\citep{scholkopf-01}, the vector $\wb$ minimizing Equation~\eqref{eq:prob_pbgd_primal} can be recovered by finding the vector $\ab$ that minimizes
\begin{align} \label{eq:prob_pbgd_dual}
F(\ab) \ = \ 
C \sum_{i=1}^m  \Phirisk\left(  y_i \frac{\sum_{j=1}^{m} \alpha_j K_{i,j}}{ \sqrt{K_{i,i}}}  \right)+
\frac{1}{2} \sum_{i=1}^{m} \sum_{j=1}^{m} \alpha_i \alpha_j K_{i,j}\,,
\end{align}
where $K$ is the kernel matrix of size $m\times m$\,.\footnote{It is non-trivial to show that the kernel trick holds when $\prior_\mathbf{0}$ and $\posterior_\wb$ are Gaussian over infinite-dimensional feature space. As mentioned by \citet{mcallester-keshet-11}, it is, however, the case provided we consider Gaussian processes as measure of distributions $\prior_\mathbf{0}$ and $\posterior_\wb$ over (infinite) $\Hcal$.
The same analysis holds for the kernelized versions of the two forthcoming domain adaptation algorithms (Section~\ref{section:kernelDA}).} 
That is, $K_{i,j} \eqdef \,k(\xb_i, \xb_j)\,.$
The gradient of $F(\ab)$ is simply given the vector $\nabla F (\ab) = (\alpha_1',\alpha_2', \ldots,\alpha_{m}')$, with 
\begin{align*} 
 \alpha'_\# \ = \ 
 \Omega\, \sum_{i=1}^m  \Phirisk\left(  y_i \frac{\sum_{j=1}^{m} \alpha_j K_{i,j}}{ \sqrt{K_{i,i}} }  \right)   \frac{y_i \,K_{i,\#}}{ \sqrt{K_{i,i}} } 
 + \sum_{j=1}^{m} \alpha_i K_{i,\#}\,,
\end{align*}
for $\# \in \{1,2,\ldots,m\,\}\,.$

\subsubsection{Improving the Algorithm Using a Convex Objective }
\label{section:pbgd3_convex}
An annoying drawback of {\small PBGD3} is that the objective function is non-convex and the gradient descent implementation needs many random restarts. 
In fact, we made extensive empirical experiments after the ones described by~\citet{germain2009pac} and saw that {\small PBGD3} achieves an equivalent accuracy (and at a fraction of the running time) by replacing the loss function $\Phirisk$ of Equations~\eqref{eq:prob_pbgd_primal} and~\eqref{eq:prob_pbgd_dual} by its convex relaxation, which is
\begin{align} \label{eq:Phic}
\Phic(x) \ &\eqdef\ 
\max \left\{\Phirisk(x),\, \frac{1}{2} - \frac{x}{\sqrt{2\pii}} \right\}\\
\ &= \nonumber
\,\left\{
\begin{array}{ll}
 \displaystyle \frac{1}{2} - \frac{x}{\sqrt{2\pii}} & \mbox{if $x\leq 0$},\\
 \Phirisk(x) & \mbox{otherwise}.
\end{array}\right.
\end{align}
The derivative of $\Phic$ at point $x$ is then $\Phic'(x) = \Phirisk'\left(\max\{0,x\}\right)$\,; In other words,
\mbox{$\Phic'(x) = {-1}/{\sqrt{2\pii}}$} if $x<0$, and $\Phirisk'(x)$ otherwise.
Figure~\ref{fig:phiclassic_a} illustrates the functions $\Phirisk$ and $\Phic$\,.
Note that the latter can be interpreted as a \emph{smooth} version the \SVM's hinge loss, $\max\{0,1-x\}$.  The toy experiment of Figure~\ref{fig:phiclassic_d} (described in the next subsection) provides another empirical evidence that the minima of $\Phirisk$ and $\Phic$ tend to coincide.

\begin{figure}[t]
	\subfloat[Loss functions for linear classifiers.]{ \label{fig:phiclassic_a}
		\raisebox{-.5\height}{\includegraphics[width=0.4\textwidth]{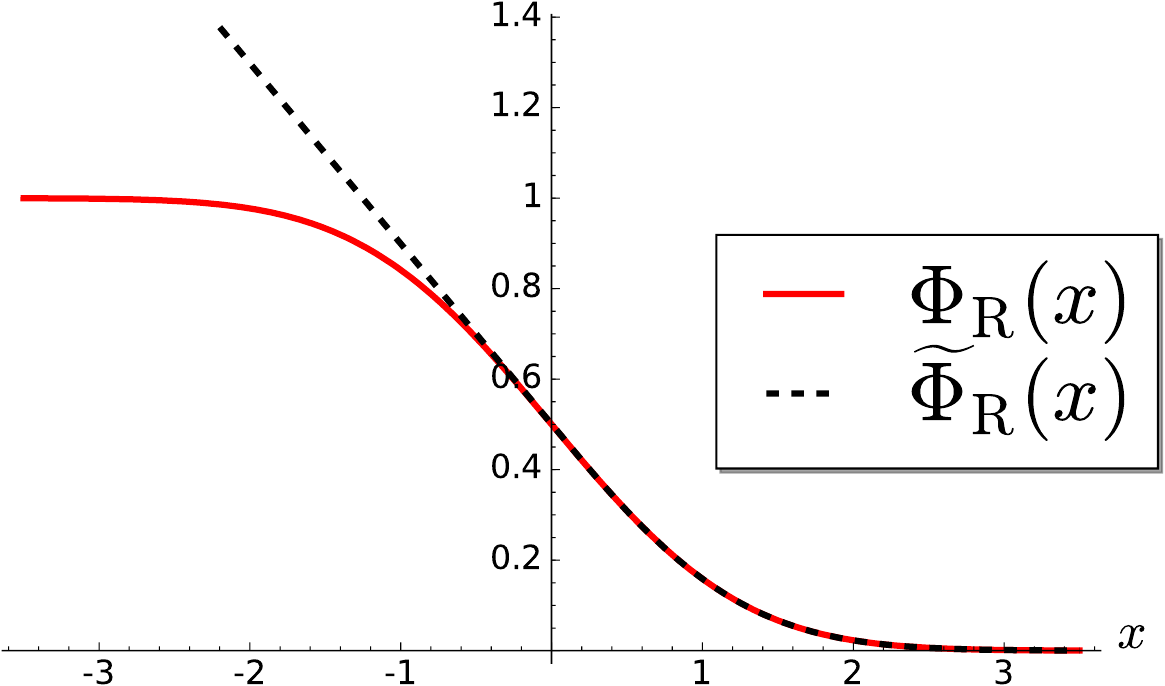}}
	}
\subfloat[Loss functions definitions and their derivatives.]{ \label{fig:phiclassic_b}
	\footnotesize \def\arraystretch{1.8}
	\begin{tabular}{|l@{\,}|l@{\,}|}
		\hline	
		\hfill\sc function & \hfill\sc derivative \\\hline
		$\Phirisk(x)=
		\frac{1}{2}  \big[1 - \Erf\big(\frac{x}{\sqrt{2}}\big) \big]$ 
		& 
		$\Phirisk'(x) = {-}\frac{1}{\sqrt{2\pii}} e^{{-}\frac{1}{2} x^2}$ 
		\\\hline
		\begin{minipage}[c]{3cm}
		$\Phic(x) =$\\
		\phantom.~~~$\max \big\{\Phirisk(x), \frac{1}{2} {-} \frac{x}{\sqrt{2\pii}} \big\}$
		\end{minipage}
		&
		\begin{minipage}[c]{2cm}
		$ \Phic'(x) =$\\  \phantom.~~~$\Phirisk'\big(\max\{0,x\}\,\big)$
		\end{minipage}
		\\\hline
	\end{tabular}
}\\
	\subfloat[Toy dataset, and the decision boundary for $\theta=0$ (matching the vertical line of Figure~(d)).]{  \label{fig:phiclassic_c}
	\raisebox{-.5\height}{\includegraphics[width=0.4\textwidth]{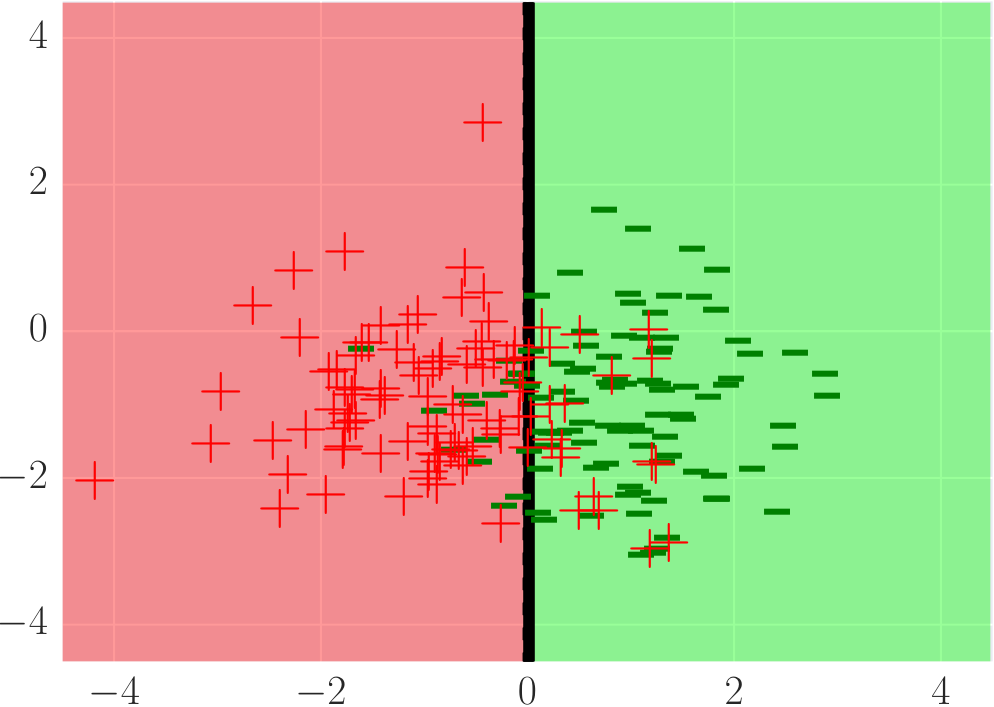}}
}
\quad
	\subfloat[Risk values according to $\theta$, for $\|\wb\|\in\{1,2,5\}$. Each dashed line shows convex counterpart of the continuous line of the same color.]{ \label{fig:phiclassic_d}
	\raisebox{-.5\height}{\includegraphics[width=0.55\textwidth]{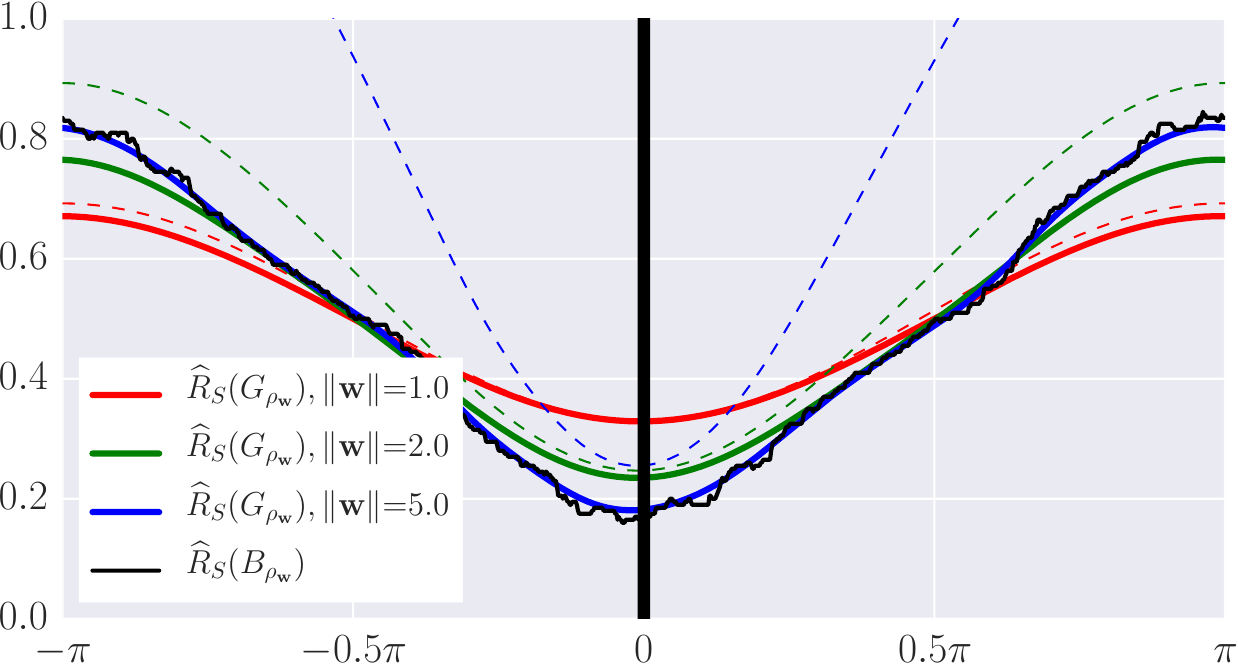}}
}

	\caption{Understanding PBGD3 supervised learning algorithm in terms of loss functions. Upper Figures (a-b) show the loss functions, and lower Figures (c-d) illustrate the behavior on a toy dataset. }
	\label{fig:phiclassic} 
\end{figure}
\subsubsection{Illustration on a Toy Dataset}
\label{section:pbgd_illustration}

To illustrate the trade-off coming into play in {\small PBGD3} algorithm (and its convexified version), we conduct a small experiment on a two-dimensional toy dataset. 
That is, we generate $100$ positive examples according to a Gaussian of mean $(-1,-1)$ and $100$ negative examples generated by a Gaussian of mean $(-1,+1)$ (both of these Gaussian have a unit variance), as shown by Figure~\ref{fig:phiclassic_c}.
We then compute the risks associated with linear classifiers $h_\wb$, with \mbox{$\wb=\|\wb\|(\cos\theta, \sin\theta)\in\Rbb^2$}. 
Figure~\ref{fig:phiclassic_d} shows the risks of three different classifiers for  $\|\wb\|\in\{1,2,5\}$, while rotating the decision boundary $\theta\in[-\pi,+\pi]$ around the origin.
The \zo-loss associated with the majority vote classifier $\RS(B_{\posterior_\wb})$ does not rely on the norm $\|\wb\|$.
However, we clearly see that probit-loss of the Gibbs classifier $\RS(G_{\posterior_\wb})$ converges to $\RS(B_{\posterior_\wb})$ as $\|\wb\|$ increases (the dashed lines correspond to the convex surrogate of the probit-loss given by Equation~\eqref{eq:Phic}). 
Thus, thanks to the specialization of to the linear classifier, the \emph{smoothness} of the surrogate loss is regularized by the norm $\|\wb\|^2$.

\section{Two New Domain Adaptation Bounds}
\label{sec:two_da_bounds}

The originality of our contribution is to theoretically design two domain adaptation frameworks suitable for the PAC-Bayesian approach. 
In Section~\ref{sec:first}, we first follow the spirit of the seminal works recalled in Section~\ref{sec:notations} by proving a similar trade-off for the Gibbs classifier.
Then in Section~\ref{sec:second}, we propose a novel trade-off based on the specificities of the Gibbs classifier that come from Equation~\eqref{eq:rde}.
Note that both results relies on the notion of expected disagreement of binary classifiers (Equation~\ref{eq:RGQGQ}). Consequently, our analysis does not directly extend to multi-class prediction and regression frameworks.

\subsection{In the Spirit of the Seminal Works}
\label{sec:first}

In the following, while the domain adaptation bounds presented in Section~\ref{sec:notations} focus on a single classifier, we first define a \mbox{$\posterior$-average} divergence measure to compare the marginals.
This leads us to derive our first domain adaptation bound.

\subsubsection{A Domains' Divergence for PAC-Bayesian Analysis}
\label{sec:distance}

As discussed in Section \ref{sec:necessity_dist}, the derivation of generalization ability in domain adaptation critically needs a divergence measure between the source and target marginals.
For the PAC-Bayesian setting, we propose a \emph{domain disagreement pseudometric}\footnote{A pseudometric $d$ is a metric for which the property  \mbox{$d(x,y)=0\ \Longleftrightarrow\ x=y$} is relaxed to \mbox{$d(x,y)=0\ \Longleftarrow\ x=y$}.} to measure the structural difference between domain marginals in terms of posterior distribution $\posterior$ over~$\Hcal$.
Since we are interested in learning a \mbox{$\posterior$-weighted}
majority vote $B_\posterior$ leading to good generalization
guarantees, we propose to follow the idea spurred by the \mbox{$C$-bound} of Equation~\eqref{eq:C-bound}: 
Given a source domain $\PS$, a target domain $\PT$, and a posterior distribution~$\posterior$, if $\RPS(G_\posterior)$ and $\RPT(G_\posterior)$ are similar, then $\RPS(B_\posterior)$ and $\RPT(B_\posterior)$ are similar when $\dDS(\Q)$ and $\dDT(\Q)$ are also similar.
Thus, the domains $\PS$ and $\PT$  are close according to $\posterior$ if the expected disagreement over the two domains tends to be close.
We then define our pseudometric as follows.
\begin{definition}
\label{def:disagreement}
Let $\Hcal$ be a hypothesis class.
For any marginal distributions $\DS$ and $\DT$ over~$\X$, any distribution $\posterior$ on $\Hcal$, the domain disagreement $\des(\DS,\DT)$ between~$\DS$ and~$\DT$ is defined by
\begin{align*}
\des(\DS,\DT)  \ &\eqdef  \ 
\Big|\,\dDT(\Q)  - \dDS(\Q) \,\Big|\\
&= \ \left| \Esp_{(h,h')\sim \posterior^2 }    \Big[ \RDT(h,h')  - \RDS(h, h') \Big] \right|.
\end{align*}
\end{definition}
Note that $\des$ is symmetric and fulfills the triangle inequality.

\subsubsection{Comparison of the {\small $\hdh$}-divergence and our domain disagreement} 

While the {\small $\hdh$}-divergence of Theorem~\ref{theo:BenDavid} is difficult to jointly optimize with the empirical source error, our empirical disagreement measure is easier to manipulate: We simply need to compute the \mbox{$\posterior$-average} of the classifiers disagreement instead of finding the pair of classifiers that maximizes the disagreement. 
Indeed, $\des(\DS,\DT)$ depends on the majority vote, which suggests that we can  directly minimize it via its empirical counterpart.
This can be done without instance reweighting, space representation changing or family of classifiers modification.
On the contrary,~$\tfrac{1}{2}d_{\hdh}(\DS,\DT)$ is a supremum over all $h \in \Hcal$ and hence, does not depend on the classifier on which the risk is considered.
Moreover, $\des(\DS,\DT)$ (the $\posterior$-average) is lower than the $\tfrac{1}{2}d_{\hdh}(\DS,\DT)$ (the worst-case).
Indeed, for every $\Hcal$ and $\posterior$ over $\Hcal$, we~have
\begin{align*}
\tfrac{1}{2}\, d_{\hdh}(\DS,\DT)  \ & =\ \sup_{\substack{(h,h')\in\mathcal{H}^2}} \left|\RDT(h,h')-\RDS(h,h')\right| \\
&\geq\  \Esp_{(h,h')\sim\posterior^2} \left| \RDT(h,h')-\RDS(h,h')\right|\\
&\geq\   \des(\DS,\DT)\,.
\end{align*}

\subsubsection{A Domain Adaptation Bound for the Stochastic Gibbs Classifier}

We now derive our first main result in the following theorem: A domain adaptation bound relevant in a PAC-Bayesian setting, and that  relies on the domain disagreement of Definition~\ref{def:disagreement}. 
\begin{theorem}
\label{theo:pbda}
Let ${\cal H}$ be a hypothesis class. We have 
\begin{align*}
\nonumber \forall \posterior&\mbox{ on }\Hcal,\  \RPT(G_\posterior)\ \leq \  \RPS(G_\posterior) +  \frac{1}{2}\des(\DS,\DT) + \lambda_\posterior\,, 
\end{align*}
where $\lambda_\rho$ is the deviation between the expected joint errors (Equation~\ref{eq:eP}) of $G_\posterior$ on the target and source domains:
 \begin{eqnarray} \label{eq:lambda_rho}
 \lambda_\posterior
 &=&
 \Big|\, \ePT(\Q) - \ePS(\Q) \,\Big|\,.
\end{eqnarray}
\end{theorem}
\smallskip

\begin{proof}
First, from Equation~\eqref{eq:rde}, we recall that, given a domain $\P$ on $\XY$ and a distribution $\posterior$ over $\Hcal$, we have
\begin{equation*} 
 \RP(G_\posterior) \ = \ \frac{1}{2} \dD(\Q) + \eP(\Q)\,.
 \end{equation*}
Therefore,
\begin{align*}
\nonumber \RPT(G_\posterior)-\RPS(G_\posterior)
& = 
\nonumber \frac{1}{2} \Big(\dDT(\Q)-\dDS(\Q)\Big) \!+\!\Big(\ePT(\posterior)-\ePS(\posterior)\Big) \\
&\leq
\nonumber \frac{1}{2} \Big|\dDT(\Q)-\dDS(\Q)\Big| +\Big|\ePT(\Q)-\ePS(\Q)\Big|  \\
&=
\frac{1}{2} \des(\DS,\DT)  + \lambda_\posterior \,. 
\end{align*}
\end{proof}

\subsubsection{Meaningful Quantities}
\label{sec:pbda_quantities}

Similar to the bounds of Theorems~\ref{theo:BenDavid} and~\ref{theo:Mansour}, our bound can be seen as a trade-off between different quantities. 
Concretely, the terms $\RPS(G_\posterior)$ and $\des(\DS,\DT)$ are akin to the first two terms of the domain adaptation bound of Theorem~\ref{theo:BenDavid}: $\RPS(G_\posterior)$ is the \mbox{$\posterior$-average} risk over $\Hcal$ on the source domain,  and $\des(\DT,\DS)$ measures the \mbox{$\posterior$-average} disagreement between the marginals but is specific to the current model depending on $\posterior$.
The other term $\lambda_\posterior$ measures the deviation between the expected joint target and source errors of $\GQ$.
According to this theory, a good domain adaptation is possible if this deviation is low. 
However, since we suppose that we do not have any label in the target sample, we cannot control or estimate it.
In practice, we suppose that $\lambda_\posterior$ is low and we neglect it.
In other words, we assume that the labeling information between the two domains is related and that considering only the marginal agreement and the source labels is sufficient to find a good majority vote. 
Another important point is that the above theorem improves the one we proposed \mbox{in~\citet{pbda}} with regard to two aspects.\footnote{More details are given in our research report~\citep{pbda_long}.}
On the one hand, this bound is not degenerated when the source and target distributions are the same or close. 
On the other hand, our result contains only half of $\des(\DS,\DT)$, contrary to our first bound proposed in~\citet{pbda}.
Finally, due to the dependence of $\des(\DT,\DS)$ and $\lambda_\posterior$ on the learned posterior, our bound is, in general incomparable with the ones of Theorems~\ref{theo:BenDavid} and~\ref{theo:Mansour}.
 However, it brings the same underlying idea: Supposing that the two domains are sufficiently related, one must look for a model that minimizes a trade-off between its source risk and a distance between the domains' marginal.

\subsection{A Novel Perspective on Domain Adaptation}
\label{sec:second}
	In this section, we introduce an original approach to upper-bound the non-estimable risk of a $\posterior$-weighted majority vote  on a target domain~$\PT$ thanks to a term depending on its marginal distribution~$\DT$, another one on a related source domain $\PS$, and a term capturing the ``volume'' of the source distribution uninformative for the target task.
We base our bound on Equation~\eqref{eq:rde} (recalled below) that decomposes the Gibbs classifier into the trade-off between the  half of the expected disagreement $\dD(\Q)$ of Equation~\eqref{eq:RGQGQ} and the expected joint error $\eP(\Q)$ of Equation~\eqref{eq:eP}:
		\begin{align} \label{eq:GibbsDE}
		 \RP(\GQ) \ & =\ \tfrac12\,\dD(\Q) + \eP(\posterior) \,.\tag{\ref{eq:rde}}
		\end{align}
	A key observation is that the \emph{voters' disagreement does not rely on labels}; we can compute $\dD(\Q)$ using the marginal distribution $\D$.
Thus, in the present domain adaptation context, we have access to $\dDT(\Q)$ even if the target labels are unknown.
However, the expected joint error can only be computed on the labeled source domain, that is what we kept in mind to define our new domain divergence.

\subsubsection{Another Domain Divergence for the PAC-Bayesian Approach}
We design a domains' divergence that allows us to link the target
joint error $\ePT(\Q)$ with the source one $\ePS(\Q)$ by reweighting the latter.
This new divergence is called the \mbox{$\betaq$-divergence} and is parametrized by a real value $q>0$\,: 
	\begin{equation} \label{eq:bq}
	\bq \ = \ \left[\,\Esp_{\exbfs}
	\left(  \frac{\PT(\xbf,y)}{\PS(\xbf,y)} \right)^q\, \right]^{\frac1q}.
	\end{equation}
	It is worth noting that  considering some $q$ values allow us to recover well-known divergences.
For instance,  choosing \mbox{$q\! =\!  2$} relates our result  to the \mbox{$\chi^2$-distance}, between the domains as \mbox{\small $\bq[2]   =  \sqrt{ \chi^2(\PT\|\PS)\!+\!1}\,.$ }
Moreover, we can  link $\bq$ to the R{\'e}nyi divergence,\footnote{For $q\geq 0$, we can easily show \mbox{$\bq  =  2^{\frac{q-1}{q} D_q(\PT\|\PS)}$}, where $D_q(\PT\|\PS)$ is the R{\'e}nyi divergence between $\PT$ and $\PS$.} which has led to generalization bounds in the specific context of importance weighting \citep{CortesMM10}. 
We denote the limit case $q \to \infty$ by
\begin{equation} \label{eq:binf}
\binf   \ =  
	\sup_{(\xbf,y)\in\scriptsupport(\PS)}
	\left(  \frac{\PT(\xbf,y)}{\PS(\xbf,y)} \right) ,
\end{equation}
with $\support(\PS)$ the support of the domain $\PS$.
The \mbox{$\betaq$-divergence} handles the input space areas where the source domain support and the target domain support $\support(\PT)$ intersect.
It seems reasonable to assume that, when adaptation is achievable, such areas are fairly large. 
However, it is likely that $\support(\PT)$ is \emph{not entirely} included in $\support(\PS)$.
We denote  $\TminusS$ the distribution of $(\xbf,y){\sim}\PT$ conditional to $(\xbf,y)\in\support(\PT){\setminus}\support(\PS)$.
Since it is hardly conceivable to estimate the joint error $\e_\TminusS(\posterior)$ without making extra assumptions, we need to define the worst possible risk for this \emph{unknown} area,
	\begin{equation}
	\label{eq:etaTS}
	\ets =    
\Pr_{(\xb,y)\sim \PT}  \Big((\xb,y)\notin \support(\PS)\Big) \ 
	\sup_{h\in\Hcal} \Risk_\TminusS(h)\,.
	\end{equation}
	Even if we cannot evaluate  $\sup_{h\in\Hcal} \Risk_\TminusS(h)$, the value of $\ets$ is necessarily lower than $\Pr_{(\xb,y)\sim \PT}\big((\xb,y)\notin \support(\PS)\big)$. 

\subsubsection{A Novel Domain Adaptation Bound}
Let us state the result underlying the novel domain adaptation perspective of this paper.
	\begin{theorem}\label{theo:new_bound_general}\label{theo:dalc}
		Let $\Hcal$ be a hypothesis space, let $\PS$ and $\PT$ respectively be the source and the target domains on $\XY$. 
		Let $q>0$ be a constant. 
		We have, 
		\begin{align*} 
		\forall \posterior\mbox{ on }\Hcal,\quad \RPT(\GQ) \, \leq \, \frac12 \,\dDT(\Q) +
		\bq{\times }
		\Big[ \ePS(\Q) \Big]^{1-\frac1q}
		 + \eta_{\PT\setminus\PS}\,,
		\end{align*}
		where
		$\dDT(\Q)$, $\ePS(\Q)$, $\bq$ and $\ets$ are respectively defined by Equations~\eqref{eq:RGQGQ}, \eqref{eq:eP}, \eqref{eq:bq} and~\eqref{eq:etaTS}.
	\end{theorem}
\begin{proof}
From Equation~\eqref{eq:GibbsDE}, we know that $\RPT(\GQ) = \tfrac12\,\dDT(\Q) + \ePT(\posterior)$.
Let us split  $\ePT(\posterior)$ in two parts:			
\begin{align} 
\nonumber\ePT(\posterior) 
\, &=  \!\!
\Esp_{\exbft} \Esp_{(h,h')\sim\posterior^2} \zoloss\left(h(\tbf),y\right)\,\zoloss\left(h'(\tbf),y\right)\\
\label{eq:encore_egal}&=\!\!
\Esp_{\exbfs}  \   \frac{\PT(\xbf,y)}{\PS(\xbf,y)}   \Esp_{(h,h')\sim\posterior^2} \zoloss\left(h(\xbf),y\right)\zoloss\left(h'(\xbf),y\right) \\[.5mm]
\label{eq:encore_egal_2} & 
\quad {}+ \!\! 
\Esp_{\exbft}     \I\left[(\tbf,y) \notin \support(\PS)\right]   \Esp_{(h,h')\sim\posterior^2} \zoloss\left(h(\tbf),y\right)\zoloss\left(h'(\tbf),y\right) .
\end{align}
{\bf (i)} On the one hand, we upper-bound the first part (Line~\ref{eq:encore_egal}) using the H{\"o}lder's inequality (see Lemma~\ref{theo:holder} in~\ref{section:tools}), with $p$ such that $\,\tfrac1p{=}1{-}\tfrac1q$\,:
	\begin{align*} 
 & \Esp_{\exbfs}  \    \frac{\PT(\xbf,y)}{\PS(\xbf,y)}   \Esp_{(h,h')\sim\posterior^2} \zoloss\left(h(\xbf),y\right)\zoloss\left(h'(\xbf),y\right) \\
\nonumber &\quad\leq \left[\,\Esp_{\exbfs} \left(  \frac{\PT(\xbf,y)}{\PS(\xbf,y)} \right)^q\, \right]^{\frac1q} \left[\Esp_{(h,h')\sim\posterior^2} \Esp_{\exbfs} \  \left[ \zoloss\left(h(\xbf), y\right)\zoloss\left(h'(\xbf), y\right) \right]^p \right]^{ \frac1p}  \\
 &\quad=	\bq \ \Big[ \ePS(\Q) \Big]^{\frac1p},
	\end{align*}
	where we have removed the exponent from expression $[\zoloss(h(\xbf), y)\zoloss(h'(\xbf), y)]^p$ without affecting its value, which is either $1$ or~$0$.
	\smallskip

\noindent {\bf (ii)} On the other hand, we upper-bound the second part (Line~\ref{eq:encore_egal_2}) by the term~$\eta_\TminusS$\,:
\begin{align*}  
&\Esp_{\exbft}     \left( \I\left[(\tbf,y) \notin \support(\PS)\right]   \Esp_{(h,h')\sim\posterior^2}  \zoloss\left(h(\tbf),y\right)\zoloss\left(h'(\tbf),y\right)  \right) \\
=& \ \left(\Esp_{\exbft}  \I[(\tbf,y) {\notin} \support(\PS)]\right) \Esp_{{(\xbf,y)\sim\TminusS}}  \Esp_{(h,h')\sim\posterior^2} \zoloss\left(h(\xbf),y\right)\,\zoloss\left(h'(\xbf),y\right)\\
=&\ \left(\Esp_{\exbft}  \I[(\tbf,y) {\notin} \support(\PS)]\right)\, \e_\TminusS(\posterior) \\
=& \ \left(\Esp_{\exbft}  \I[(\tbf,y) {\notin} \support(\PS)]\right)\Big(\Risk_\TminusS(G_\posterior) - \tfrac12{\rm d}_{\TminusS}(\posterior) \Big)\\
\leq& \ \left(\Esp_{\exbft}  \I[(\tbf,y) {\notin} \support(\PS)]\right) \sup_{h\in\Hcal} \Risk_\TminusS(h) \ =\ \eta_\TminusS\,.
\end{align*}
\end{proof}

Note that the bound of Theorem~\ref{theo:new_bound_general} is reached whenever the domains are equal ($\PS  =  \PT$).
Thus, when adaptation is not necessary, our analysis is still sound and  non-degenerated: 
\begin{align*} 
\RPS(\GQ) \ =\ \RPT(\GQ) \ &\leq\  \tfrac12\, \dDT(\Q) +  1\times\left[\ePS(\Q)\right]^1+0\\
&=\  \tfrac12\, \dDS(\Q) +  \ePS(\Q) =\ \RPS(\GQ) \,.
		\end{align*}

\subsubsection{Meaningful Quantities and Connection with Some Domain Adaptation Assumptions }
\label{sec:dalc_quantities}
Similarly to the previous results recalled in Section~\ref{sec:notations}, our domain adaptation theorem bounds the target risk by a sum of three terms. 
However, our approach breaks the problem into \emph{atypical} quantities:
\begin{enumerate}[(i)]
\item The expected disagreement $\dDT(\posterior)$ captures \emph{second degree} information about the target domain (without any label).
\item The \mbox{$\betaq$-divergence} $\bq$ is not an additional term: It weighs the influence of the expected joint error $\ePS(\posterior)$ of the source domain; the parameter~$q$ allows us to consider different relationships between $\bq$ and $\ePS(\Q)$. 
\item The term $\ets$ quantifies the worst feasible target error on the regions where the source domain is uninformative for the target one.
In the current work, we assume that this area is small.
\end{enumerate}

We now establish some connections with existing common domain
adaptation assumptions in the literature. Recall that in order to characterize which domain adaptation task may be \emph{learnable}, \citet{bendavid-12} presented three \emph{assumptions that can help domain adaptation}. 
Our Theorem~\ref{theo:new_bound_general} does not rely on these assumptions, and remains valid in the absence of these assumptions, but they can be interpreted in our framework as discussed below. 

\paragraph{On the covariate shift} 
A domain adaptation task fulfills the \emph{covariate shift} assumption \citep{covariateshift} if the source and target domains only  differ in their marginals according to the input space, {\it i.e.,} $\PT_{\Y|\xb}(y) = \PS_{\Y|\xb}(y)$. 
In this scenario, one may estimate the values of $\bqx$, and even $\ets$, by using unsupervised density estimation methods.
Interestingly, with the additional assumption that the domains share the same support, we have \mbox{$\ets = 0$}. Then from Line~\eqref{eq:encore_egal}  we obtain 
$$\displaystyle \RPT(\GQ) \ =\ \tfrac12 \dDT(\Q) + \Esp_{\xbfs} \frac{\DT(\xbf)}{\DS(\xbf)} \Esp_{h\sim\posterior} \Esp_{h' \sim\posterior} \zoloss\left(h(\xbf),y\right)\zoloss\left(h'(\xbf),y\right) ,$$
which suggests a way to correct the \emph{shift} between the domains by reweighting the labeled source distribution, while considering the information from the target disagreement.

\paragraph{On the weight ratio} 
 The \emph{weight ratio} \citep{bendavid-12} of source and target domains, with respect to a collection of input space subsets $\Bcal\subseteq 2^\X$, is given by
$$ C_\Bcal(\PS, \PT) \ = \ \inf_{\substack{b\in\Bcal,\, \DT(b)\neq 0 }} \,\frac{\DS(b)}{\DT(b)}\,.$$
When $C_\Bcal(\PS, \PT)$ is bounded away from $0$, adaptation should be achievable under covariate shift.
In this context, and when \mbox{$\support(\PS) = \support(\PT)$}, the limit case of  $\binf$ is equal to the inverse of the \emph{pointwise weight ratio} obtained by letting $\Bcal = \{\{\xbf\}:\xbf\in\X\}$ in $C_\Bcal(\PS, \PT)$.
Indeed, both $\beta_q$ and $C_\Bcal$ compare the density of source and target domains, but provide distinct strategies to relax the pointwise weight ratio; the former by lowering the value of~$q$ and the latter by considering larger subspaces~$\Bcal$.

\paragraph{On the cluster assumption} 
A target domain fulfills the \emph{cluster assumption} when examples of the same label belong to a common ``area'' of the input space, and the differently labeled ``areas'' are well separated by \emph{low-density regions} (formalized by the \emph{probabilistic Lipschitzness} \cite{urner-11}). 
Once specialized to linear classifiers, $\dDT(\posterior)$ behaves nicely in this context (see Section~\ref{sec:pbda}).

\paragraph{On representation learning}
The main assumption underlying our domain adaptation algorithm exhibited in Section~\ref{sec:pbda} is that the support of the target domain is mostly included in the support of the source domain, \ie, the value of the term $\ets$ is small.
In situations when $\TminusS$ is sufficiently large to prevent proper adaptation, one could try to reduce its volume  while taking care to preserve a good compromise between $\dDT(\Q)$ and $\ePS(\Q)$, using a \emph{representation learning} approach, \ie, by projecting source and target examples into a new common input space, as done for example by \citet{Chen12,ganin-16} (see~\cite{wang2018deep} for a survey of representation learning approaches for domain adaptation vision tasks).

\subsection{Comparison of the Two Domain Adaptation Bounds}
Since they rely on different approximations, the gap between the bounds of Theorems~\ref{theo:pbda} and \ref{theo:new_bound_general} varies according to the context.
As presented in Sections~\ref{sec:pbda_quantities} and~\ref{sec:dalc_quantities}, the main difference between our two bounds lies in the estimable terms, from which we will derive algorithms in Section~\ref{sec:pbda}.
In Theorem~\ref{theo:new_bound_general}, the non-estimable terms are the domains' divergence $\bq$ and the term $\ets$. 
Contrary to the non-controllable term $\lambda_\posterior$ of Theorem~\ref{theo:pbda}, these terms do not depend on the \emph{learned} posterior distribution~$\posterior$: 
For every $\posterior$ on $\Hcal$, $\bq$ and $\ets$ are constant
values measuring the relation between the domains for the considered task. 
Moreover, the fact that the \mbox{$\betaq$-divergence} is not an additive term but a multiplicative one (as opposed to {\small$\dis(\DS,\DT)  +  \lambda_\posterior$} in Theorem~\ref{theo:pbda}) is a contribution of our new perspective.
Consequently, $\bq$ can be viewed as a hyperparameter allowing us to tune the trade-off between the target voters' disagreement and the source joint error. 
Experiments of Section~\ref{sec:expe} confirm that this hyperparameter can be successfully selected.

Note that,
when $\ePT(\Q) \geq \ePS(\Q)$, 
we can upper-bound the term $\lambda_\posterior$ of Theorem~\ref{theo:pbda} by using the same trick as in Theorem~\ref{theo:new_bound_general} proof. 
This leads to
\begin{align*}
\ePT(\Q) \,\geq\, \ePS(\Q)  \quad \Longrightarrow \quad \lambda_\Q\ &=\  \ePT(\Q) - \ePS(\Q) \\
&\leq\ \bq{\times} \big[ \ePS(\Q) \big]^{1-\frac1q}  + \eta_{\PT\setminus\PS} - \ePS(\Q)\,.
\end{align*}
Thus, in this particular case, we can rewrite Theorem~\ref{theo:pbda} statement as for all $\posterior$ on $\Hcal$, we have
\begin{align*}
\RPT(G_\posterior) \leq   \RPS(G_\posterior) + \frac{1}{2}\des(\DS,\DT) +
\bq{\times} \big[ \ePS(\Q) \big]^{1-\frac1q}  - \ePS(\Q) 
{+} \eta_{\PT\setminus\PS}\,.
\end{align*}
It turns out that, if $\dDT(\Q)  \geq \dDS(\Q)$
in addition to $\ePT(\Q) \geq \ePS(\Q)$,
the above statement reduces to the one of Theorem~\ref{theo:new_bound_general}. In words,
this occurs in the very particular case where the target disagreement
and the target expected joint error are both greater than their source counterparts,
which may be interpreted as a rather favorable situation. 
However, Theorem~\ref{theo:new_bound_general} is tighter in all other cases. 
This highlights that introducing absolute values in Theorem~\ref{theo:pbda} proof leads to a crude approximation.
Remember that we have first followed this path to stay aligned with classical domain adaptation analysis, but our second approach leads to a more suitable analysis in a PAC-Bayesian context.
Our experiments of Subsection~\ref{section:dalc_illustration} illustrate this empirically, once the domain adaptation bounds are converted into PAC-Bayesian generalization guarantees for linear classifiers.

\section{PAC-Bayesian Generalization Guarantees}
\label{sec:pb_da_bounds}
To compute our domain adaptation bounds, one needs to know the distributions $\PS$ and~$\DT$, which is never the case in real life tasks. 
 PAC-Bayesian theory provides tools to convert the bounds of Theorems~\ref{theo:pbda} and~\ref{theo:new_bound_general} into generalization bounds on the target risk computable from a pair of source-target samples $(S,T){\sim}(\PS)^\ms {\times} (\DT)^\mt$. 
To achieve this goal, we first provide generalization guarantees for the terms involved in our domain adaptation bounds: $\dDT(\Q)$, $\ePS(\Q)$, and $\des(\DS,\DT)$.
These results are presented as corollaries of Theorem~\ref{theo:catoni_genral} below, that generalizes the PAC-Bayesian of~\citet{catoni2007pac} (see Theorem~\ref{thm:pacbayescatoni} in Section~\ref{sec:threepb}) to arbitrary loss functions.
Indeed, Theorem~\ref{theo:catoni_genral}, with $\ell(h,\xb,y)  = \zoloss\left(h(\xb),y\right)$ and Equation~\eqref{eq:RGQ}, gives the usual bound on the Gibbs risk. \\
	
Note that the proofs of Theorem~\ref{theo:catoni_genral} (deferred in~\ref{proof:catoni_general}) and Corollary~\ref{theo:catoni_new}  (below) reuse techniques from related results presented in \citet{graal-neverending}.
Indeed, PAC-Bayesian bounds on $\dDT(\Q)$ and $\ePS(\Q)$ appeared in the latter, but under different forms. 
	\begin{theorem}
		\label{theo:catoni_genral}
		For any domain $\P$ over $\X\times Y$, for any set of hypotheses $\Hcal$, any prior  $\prior$ over~$\Hcal$, any loss $\ell:\Hcal{\times}\X{\times}\Y \to [0,1]$, any real number $\alpha > 0$, with a probability at least $1{-}\delta$ over the random choice of $\{(\xb_i,y_i)\}_{i=1}^m {\sim} (\P)^m$, we have, for all $\posterior$ on $\Hcal$,
		\begin{align*}
		\Esp_{\exbf\sim\P}  \Esp_{h\sim\posterior} \ell(h, \xb,y)\ \leq\ \frac{\alpha}{1{-}e^{-\alpha}} \left[\frac{1}{m}\sum_{i=1}^m \Esp_{h\sim\posterior}   \ell(h,\xb_i,y_i) + \frac{\KL(\posterior\|\prior) + \ln \frac{1}{\delta}}{ m\times \alpha}\right].
		\end{align*}
	\end{theorem}

	We now exploit Theorem~\ref{theo:catoni_genral} to obtain generalization guarantees on the expected disagreement, the expected joint error, and the domain disagreement.
	In Corollary~\ref{theo:catoni_new} below, we are especially interested in the possibility of controlling the trade-off---between the empirical estimate computed on the samples and the complexity term $\KL(\posterior\|\prior)$---with the help of parameters $a$, $b$ and $c$.	
	\begin{corollary}
		\label{theo:catoni_new}
		For any domains $\PS$ and $\PT$ over $\X\times\Y$, any set of voters $\Hcal$, any prior $\prior$ over~$\Hcal$, any $\delta\in (0,1]$, any real numbers $a>0$, $b > 0$ and $c>0$, we have\\
		--- with a probability at least $1{-}\delta$ over  $T\sim(\DT)^{\mt}$, 
		\begin{align*}
		\forall \posterior \mbox{ on }\Hcal,\ \dDT(\Q)\, \leq\, \frac{c}{1{-}e^{-c}}  \left[\dT(\Q)  +    \frac{2\,\KL(\posterior\|\prior)  +  \ln \frac{1}{\delta}}{\mt\times c}\right] ,
		\end{align*}
		--- with a probability at least $1{-}\delta$ over  $S\sim(\PS)^{\ms}$,
		\begin{align*}
		\forall \posterior \mbox{ on }\Hcal, \	\ePS(\Q)\, \leq\, \frac{b}{1{-}e^{-b}}  \left[\eS(\Q)  +    \frac{2\,\KL(\posterior\|\prior)  +  \ln \frac{1}{\delta}}{\ms\times b}\right] ,
		\end{align*}
		--- with a probability at least $1{-}\delta$ over $S \times  T  \sim (\DS \times  \DT)^m $,
		\begin{align*}
		\forall \posterior \mbox{ on }\Hcal, \
\des(\DS,\DT)\, \leq\,   \frac{2\,a }{1 -e^{-2a}}  \left[ \desST  +  \frac{2\,\KL(\posterior\|\prior)  +  \ln  \frac{2}{\delta}}{m\times a} + 1\right] - 1\,,
		\end{align*}
where $\dT(\Q)$, $\eS(\posterior) $, and $\desST$ are the empirical estimations of the target voters' disagreement, the source joint error, and the domain disagreement.
	\end{corollary}
	
	\begin{proof}
Given $\prior$ and $\posterior$ over $\Hcal$, we consider a new prior $\prior^2$ and a new posterior $\posterior^2$, both over $\Hcal^2$, such that: \mbox{$\forall\, h_{ij}  =  (h_i,h_j) \in  \Hcal^2,\ \prior^2(h_{ij})  =  \prior(h_i)\prior(h_j)$,} and \mbox{$\posterior^2(h_{ij})  =  \posterior(h_i)\posterior(h_j)$}.
Thus, $\KL(\posterior^2\|\prior^2)  =  2\, \KL(\posterior\|\prior)$ (see Lemma~\ref{lem:2KL} in~\ref{section:tools}).
Let us define four new loss functions for a ``paired voter'' $h_{ij}  \in  \Hcal^2$:
\begin{align*}
\ell_d(h_{ij}, \xb, y)\ &= \ \zoloss\left(h_i(\xbf),h_j(\xbf)\right),\\
\ell_e(h_{ij}, \xb, y)\ &= \ \zoloss\left(h_i(\xbf), y\right) \times  \zoloss\left(h_j(\xbf), y\right),\\
 \ell_{d^{(1)}}(h_{ij}, (\xbf^s, \xbf^t), \cdot)\ & =\ \frac{1+\zoloss\left(h_i(\xbf^s),h_j(\xbf^s)\right) - \zoloss\left(h_i(\xbf^t),h_j(\xbf^t)\right)}{2}\,,\\
\ell_{d^{(2)}}(h_{ij}, (\xbf^s, \xbf^t), \cdot)\ & =\ \frac{1+\zoloss\left(h_i(\xbf^t),h_j(\xbf^t)\right) - \zoloss\left(h_i(\xbf^s),h_j(\xbf^s)\right)}{2}\,.
\end{align*}
Thus, from Theorem~\ref{theo:catoni_genral}:
\begin{itemize}
\item The bound on $\dDT(\Q)$ is obtained  with $\ell\eqdots \ell_d$, and Equation~\eqref{eq:RGQGQ};
\item The bound on $\ePS(\Q)$ is similarly obtained with $\ell\eqdots \ell_e$, and  Equation~\eqref{eq:eP};
\item The bound on $\des(\DS,\DT)$ is obtained with  $\ell\eqdots \ell_{d^{(1)}}$, by upper-bounding
\begin{equation*}
d^{(1)} = \dDS(\Q) - \dDT(\Q) = 2 \Esp_{h_{ij}\sim\posterior^2} \Esp_{(\xbf^s, \xbf^t) \sim \DS \times  \DT} \ell_{d^{(1)}} (h_{ij}, (\xbf^s, \xbf^t), \cdot) - 1\,,
\end{equation*}
from its empirical counterpart
\begin{equation*}
\widehat{d}^{(1)} = \dS(\Q) - \dT(\Q) = \frac2m \Esp_{h_{ij}\sim\posterior^2} \sum_{k=1}^m \ell_{d^{(1)}} (h_{ij}, (\xbf_k^s, \xbf_k^t), \cdot) - 1\,.
\end{equation*}
We then have, 
with probability $1 - \frac\delta2$ over the choice of $S\times T\sim(\DS\times \DT)^m$, 
\begin{equation*}
\frac{|d^{(1)}|+1}{2}\ \leq \  \frac{a}{ 1 - e^{-2a} } \left[|\widehat{d}^{(1)}| +1+ \frac{2\,\KL(\posterior\|\prior)  +  \ln  \frac{2}{\delta}}{m\times a} \right] \,.
\end{equation*}
In turn, with $\ell\eqdots \ell_{d^{(2)}}$ we bound $d^{(2)} = \dDT(\Q) - \dDS(\Q)$ by its empirical counterpart $\widehat{d}^{(2)} = \dT(\Q) - \dS(\Q)$, with probability $1 -\frac\delta2$ over the choice of $S\times T\sim(\DS\times \DT)^m$, 
\begin{equation*}
\frac{|d^{(2)}|+1}{2}\ \leq \  \frac{a}{ 1 - e^{-2a} } \Bigg[|\widehat{d}^{(2)}| +1+ \frac{2\,\KL(\posterior\|\prior)  +  \ln  \frac{2}{\delta}}{m\times a} \Bigg] \,.
\end{equation*}
Finally, by the union bound, with probability $1 - \delta$, we have
 \begin{align*}
\des(\DS,\DT)\, &= \, \max\Big\{d^{(1)}, d^{(2)}\Big\}\\
&\leq\,  
  \frac{2a}{1 -e^{-2a}} \left[ \desST  +  \frac{2\,\KL(\posterior\|\prior)  +  \ln  \frac{2}{\delta}}{m\times a} + 1\right]\!{-} 1\,,
\end{align*}
\end{itemize}
 and we are done.
\end{proof}

\medskip

The following bound is based on the above Catoni's approach for our domain adaptation bound of Theorem~\ref{theo:pbda} and corresponds to the one from which we derive---in Section~\ref{sec:pbda}---\PBDA our first algorithm for PAC-Bayesian domain adaptation.

\begin{theorem}
 \label{theo:pacbayesdabound_catoni_bis}
 For any domains $\PS$ and $\PT$ (resp.~with marginals $\DS$ and $\DT$) over $\X \times   Y$, any set of hypotheses  $\Hcal$,  any prior distribution $\prior$ over $\Hcal$, any $\delta \in (0,1]$, any real numbers $\omega > 0$ and $a > 0$,  with a probability at least $1-\delta$ over the choice of $S \times  T  \sim (\PS {\times}  \DT)^m $, for every posterior distribution $\posterior$ on $\Hcal$, we have
 \begin{align*} 
\RPT(G_\posterior)   \leq  \omega' \RS(G_\posterior)  {+}  a'\tfrac{1}{2} \desST {+}  \left( \tfrac{\omega'}{\omega} {+} \tfrac{a'}{a} \right)  \frac{\KL(\posterior\|\prior{+}\ln\frac{3}{\delta}}{m} + \lambda_\posterior + \tfrac{1}{2} (a' {-}  1)
   \,,
 \end{align*}
where  $\RS(G_\posterior)$ and $\desST$ are the empirical estimates of the target risk and the domain disagreement;
$\lambda_\rho$ is defined by Equation~\eqref{eq:lambda_rho}; $\displaystyle \omega'\eqdef\tfrac{\omega}{1 -e^{-\omega}}$ \, and \, $\displaystyle a'\eqdef \tfrac{2a}{1 -e^{-2a}}$\,.
\end{theorem}
\begin{proof}
In Theorem~\ref{theo:pbda}, we replace $\RPS(G_\posterior)$ and $\des(\DS,\DT)$ by their upper bound, obtained from Theorem~\ref{thm:pacbayescatoni} and Corollary~\ref{theo:catoni_new}, with $\delta$ chosen respectively as $\frac{\delta}{3}$ and $\frac{2\delta}{3}$.
In the latter case, we use 
$$2\,\KL(\posterior\|\prior) + \ln\tfrac{2}{2\delta/3} \,=\, 2\,\KL(\posterior\|\prior) +\ln\tfrac{3}{\delta} \,<\,  2\left( \KL(\posterior\|\prior) +\ln\tfrac{3}{\delta} \right).$$
\end{proof}

\medskip

We now derive a PAC-Bayesian generalization bound for our second domain adaptation bound of Theorem~\ref{theo:new_bound_general} from which we derive---in Section~\ref{sec:pbda}---our second algorithm for PAC-Bayesian domain adaptation \algo. 
For algorithmic simplicity, we deal with Theorem~\ref{theo:new_bound_general} when $q{\to}\infty$. 
Thanks to Corollary~\ref{theo:catoni_new}, we obtain the following generalization bound defined with respect to	the empirical estimates of the target disagreement and the source joint error.
	\begin{theorem}
		\label{theo:catoni_new_general}
		For any domains $\PS$ and $\PT$ over $\X\times \Y$, any set of voters $\Hcal$, any prior $\prior$ over $\Hcal$, any $\delta\in (0,1]$, any real numbers $b  >  0$ and $c > 0$, with a probability at least $1{-}\delta$ over the choices of $S{\sim}(\PS)^\ms$ and $ T {\sim} (\DT)^\mt$, for every posterior distribution $\posterior$ on $\Hcal$, we have
		\begin{align*}
		\RPT(\GQ)\ \leq\ c'\,\tfrac12\,\dT(\Q)  + b'\,\eS(\posterior) + \ets + \left(\frac{c'}{\mt\, c} {+} \frac{b'}{\ms\, b} \right)     \left(2\,\KL(\posterior\|\prior) {+} \ln \tfrac{2}{\delta}\right) ,
		\end{align*}
where $\dT(\Q)$ and $\eS(\posterior) $ are the empirical estimations of the target voters' disagreement and the source joint error, and $\displaystyle b'=\tfrac{b}{1-e^{-b}}\,\binf$, and $\displaystyle c'=\tfrac{c}{1-e^{-c}}$.
	\end{theorem}
	\begin{proof}
We bound  separately $\dDT(\Q)$ and $\ePS(\Q)$ using Corollary~\ref{theo:catoni_new} (with probability $1{-}\frac{\delta}{2}$ each), and then combine the two upper bounds according to Theorem~\ref{theo:new_bound_general}.
	\end{proof}

From an optimization perspective, the problem suggested by the bound of Theorem~\ref{theo:catoni_new_general} is much more convenient to minimize than the PAC-Bayesian bound derived in Theorem~\ref{theo:pacbayesdabound_catoni_bis}.	
The former is \emph{smoother} than the latter: The absolute value related to the domain disagreement $\dis(\DS,\DT)$ disappears in benefit of the domain divergence $\binf$, which is constant and can be considered as an hyperparameter of the algorithm.
Additionally, Theorem~\ref{theo:pacbayesdabound_catoni_bis} requires equal source and target sample sizes while Theorem~\ref{theo:catoni_new_general} allows \mbox{$\ms\neq\mt$}.
Moreover, for algorithmic purposes, we ignore the \mbox{$\posterior$-dependent} non-constant term $\lambda_\posterior$ of Theorem~\ref{theo:pacbayesdabound_catoni_bis}.
In our second analysis, such compromise is not mandatory in order to apply the theoretical result to real problems, since the non-estimable term $\ets$ is constant and does not depend on the learned $\posterior$.
Hence, we can neglect $\ets$ without any impact on the optimization problem described in the next section. 
Besides, it is realistic to consider $\ets$ as a small quantity in situations where the source and target supports are similar.

\section[PAC-Bayesian DA Learning of Linear Classifiers]{PAC-Bayesian Domain Adaptation Learning of Linear Classifiers} 
\label{sec:dapbgd}
\label{sec:pbda}

In this section,  we design two learning algorithms for domain adaptation\footnote{The code of our algorithms are available on-line. More details are given in Section~\ref{sec:expe}.} inspired by the PAC-Bayesian learning algorithm of~\citet{germain2009pac}.
That is, we adopt the specialization of the PAC-Bayesian theory to linear classifiers described in Section~\ref{sec:pbgd}. 
The taken approach is the one privileged in numerous PAC-Bayesian works \citep[\eg,][]{Langford02,AmbroladzePS06,mcallester-keshet-11,Parrado-Hernandez12,germain2009pac,pbda}, as it makes the risk of the linear classifier $h_\wb$ and the risk of a (properly parametrized) majority vote coincide, while in the same time promoting large margin classifiers.

\subsection{Domain and Expected Disagreement, Joint Error of Linear Classifiers}

Let us consider a prior $\prior_\mathbf{0}$ and a posterior $\posterior_\wb$ that are spherical Gaussian distributions over a space of linear classifiers, exactly as defined in Section~\ref{sec:pbgd}.
We seek to express the \emph{domain disagreement} $\desw(\DS,\DT)$, \emph{expected disagreement} $\dD(\posterior_\wb)$ and the \emph{expected joint error}~$\eP(\posterior_\wb)$.\\
\noindent First, for any marginal $\D$, the expected disagreement for linear classifiers is
\begin{align}
\nonumber \dD(\posterior_\wb)
\nonumber =& \Esp_{\xbf\sim \D}  \Esp_{(h,h')\sim \posterior_\wb^2} \zoloss\big( h(\xb), h'(\xb) \big) \\
\nonumber =&  \Esp_{\xbf\sim \D}  \Esp_{(h,h')\sim \posterior_\wb^2} \I [h(\xb)\neq h'(\xb)]\\
\nonumber =&  \Esp_{\xbf\sim \D}  \Esp_{(h,h')\sim \posterior_\wb^2} \!\! \Big(\I [h(\xb)\!=\!1] \, \I [h'(\xb)\!=\!-1] + \I [h(\xb)\!=\!-1] \, \I [h'(\xb)\!=\!1]\Big)\\
\nonumber =&  \ 2 \Esp_{\xbf\sim \D}\  \Esp_{(h,h')\sim \posterior_\wb^2} \I [h(\xb)=1] \, \I [h'(\xb)=-1]\\
\nonumber =&  \ 2 \Esp_{\xbf\sim \D}\  \Esp_{h\sim \posterior_\wb} \I [h(\xb)=1] \,  \Esp_{h'\sim \posterior_\wb}\I [h'(\xb)=-1]\\
\nonumber =& \ 2  \Esp_{\xbf\sim \D} \Phirisk\left( \frac{\wb \cdot \xb}{\|\xb\|}  \right) \  \Phirisk\left(  - \frac{\wb \cdot \xb}{\|\xb\|}  \right)\\
\label{eq:RGQGQ_lin} =&  \Esp_{\xbf\sim \D} \Phidis\left( \frac{\wb \cdot \xb}{\|\xb\|}\right), 
\end{align}
\noindent where
\begin{equation} \label{eq:phidis}
\Phidis(x)\ \eqdef \ 2\,\Phirisk(x)\,\Phirisk(-x)\,.
\end{equation}
Thus, the domain disagreement for linear classifiers is
\begin{eqnarray}
\nonumber\desw(\DS,\DT)  &=& \Big| \,  \dDS(\posterior_\wb) - \dDT({\posterior_\wb})\,\Big| \\
&=& 
\left| 
\Esp_{\sbf\sim \DS} \Phidis \left(  \frac{\wb \cdot \sbf}{\|\sbf\|}  \right)    -  \Esp_{\tbf\sim \DT}
\Phidis \left(  \frac{\wb \cdot \tbf}{\|\tbf\|}  \right)  \right|. 
\label{eq:dis_lin}
\end{eqnarray}

\noindent Following a similar approach, the expected joint error is, for all $\wb\in\Rbb$,
\begin{eqnarray}
\nonumber	\ep(\posterior_\wb) &=&	\Esp_{\exbf\sim\P} \Esp_{h\sim\posterior_\wb} \Esp_{h'\sim\posterior_\wb}  \zoloss\left(h(\xbf),y\right)\times \zoloss\left(h'(\xbf),y\right)\\
\nonumber &=& \Esp_{\exbf\sim\P}  \Esp_{h\sim\posterior_\wb}   \zoloss\left(h(\xbf), y\right) \Esp_{h'\sim\posterior_\wb} \zoloss\left(h'(\xbf), y\right)\\[-1.5mm]
	&=& \Esp_{\exbf\sim\P} \Phierr \left(  y\, \frac{\wb \cdot \xbf}{\|\xbf\|}  \right),
	\label{eq:eP_lin}
\end{eqnarray}
with  
\begin{equation} \label{eq:phierr}
\Phierr(x) \ =\  \big[\Phirisk(x)\big]^2\,.
\end{equation}
Functions $\Phierr$ and $\Phidis$ defined above can be interpreted as loss functions for linear classifiers (illustrated by Figure~\ref{fig:phida_a}).

\subsection{Domain Adaptation Bounds}
Theorems~\ref{theo:pbda} and~\ref{theo:new_bound_general} (when $q{\to}\infty$) specialized to linear classifiers give the two following corollaries. We recall that  $\RPT(h_\wb) = \RPT(B_{\posterior_\wb})  \leq 2\,\RPT(G_{\posterior_\wb})$. 
\begin{corollary}\label{cor:new_bound_pbda_linear}
 Let $\PS$ and $\PT$ respectively be the source and the target domains on $\XY$. For all $\wb\in \Rbb$, we have
\begin{align*}
 \RPT(h_\wb) \, \leq \,  2 \,\RPS(G_{\posterior_\wb}) + \desw(\DS,\DT) + 2 \lambda_{\posterior_\wb}  \,,	
\end{align*}
where $\desw(\DS,\DT)$ and $\lambda_{\posterior_\wb}$ are respectively defined by Equations~\eqref{eq:dis_lin} and~\eqref{eq:lambda_rho}.
\end{corollary}
\medskip

\begin{corollary}\label{cor:new_bound_linear}
 Let $\PS$ and $\PT$ respectively be the source and the target domains on $\XY$. For all $\wb\in \Rbb$, we have
\begin{align*}
 \RPT(h_\wb) \, \leq \, \dDT(\Q_\wb) + 2\,\binf \times  \ePS(\Q_\wb) + 2\,\ets  \,,	
\end{align*}
where $\dDT(\posterior_\wb)$, $\ePS(\posterior_\wb)$, $\binf$ and $\ets$ are respectively defined by Equations~\eqref{eq:RGQGQ_lin}, \eqref{eq:eP_lin}, \eqref{eq:binf} and~\eqref{eq:etaTS}.
\end{corollary}

For fixed values of $\binf$ and $\ets$, the target risk $\RPT(h_\wb)$ is upper-bounded by a \mbox{$\betainf$-weighted} sum of two losses.
 The  expected \mbox{$\Phierr$-loss} (\ie, the joint error) is computed on the (labeled) source domain; it aims to label the source examples correctly, but is more permissive on the required margin than the $\Phi$-loss (\ie, the Gibbs risk).
The expected \mbox{$\Phidis$-loss} (\ie, the disagreement) is computed on the target (unlabeled) domain; it promotes large \emph{unsigned} target margins.
Thus, if a target domain fulfills the \emph{cluster assumption}
(described in Section~\ref{sec:dalc_quantities}), 
  $\dDT(\posterior_\wb)$ will be low when the decision boundary crosses a low-density region between the homogeneous labeled clusters. 
Hence, Corollary~\ref{cor:new_bound_linear} reflects that some source errors may be allowed if, doing so, the separation of the target domain is improved.
Figure~\ref{fig:phida_a} leads to an insightful geometric interpretation of the two domain adaptation trade-off
promoted by Corollaries~\ref{cor:new_bound_pbda_linear} and \ref{cor:new_bound_linear}.

\subsection{Generalization Bounds and Learning Algorithms}

\subsubsection{First Domain Adaptation Learning Algorithm ({\small PBDA}).}

\noindent
Theorem~\ref{theo:pacbayesdabound_catoni_bis} specialized to linear classifiers gives the following.
\begin{corollary}
\label{cor:gen_bound linear}
For any domains $\PS$ and $\PT$ over $\X \times\Y$, any $\delta\in(0,1]$, any  $\omega > 0$ and $a>0$, with a probability at least $1{-}\delta$ over the choices of $S\sim(\PS)^m$ and $ T\sim(\DT)^m$, we have, for all $\wb\in\Rbb$\,,
		\begin{align*}
		\RPT(h_\wb) & \leq\, 2 \omega' \RS(G_{\posterior_\wb}) {+}  a' \deswST {+} 2 \lambda_{\posterior_\wb} {+} 2 \left( \tfrac{\omega'}{\omega} {+} \tfrac{a'}{a} \right)  \frac{\|\wb\|^2{+}\ln\frac{3}{\delta}}{m}  +\! (a' {-}  1)\,,
		\end{align*}
where $\RS(G_{\posterior_\wb})$ and $\deswST$, are the empirical estimates of the target risk and the domain disagreement\,; $\lambda_{\rho_\wb}$ is obtained using Equation~\eqref{eq:lambda_rho}; $\omega'\eqdef\frac{\omega}{1 -e^{-\omega}}$, \, and \, $ a'\eqdef \frac{2a}{1 -e^{-2a}}$\,.
	\end{corollary} 

Given a source sample $S = \{(\xb^s_i, y^s_i)\}_{i=1}^m$ and a target sample $T = \{(\xb^t_i)\}_{i=1}^m$, we focus on the minimization of the bound given by Corollary~\ref{cor:gen_bound linear}.
We work under the assumption that the term $\lambda_{\posterior_\wb}$ of the bound is negligible.
Thus, the posterior distribution $\posterior_\wb$ that minimizes the bound on $\RPT(h_\wb)$ is the same that minimizes
\begin{align} \label{eq:probpbda2}
&\Omega\,m\, \RS(G_{\posterior_\wb})  +  A \,m\, \deswST +  \KL(\posterior_\wb \| \prior_\mathbf{0})\\
&\quad=\, \Omega \sum_{i=1}^m  \Phirisk\!\left(  y^s_i \frac{\wb \cdot \xb^s_i}{\|\xb^s_i\|}  \right)   +  
A  \left| \sum_{i=1}^m  \left[ \Phidis\! \left(  \frac{\wb \cdot \xb^s_i}{\|\xb^s_i\|}  \right)    -  \Phidis \!\left(  \frac{\wb \cdot \xb^t_i}{\|\xb^t_i\|}  \right) \right] \right|
  +   \frac{1}{2}\|\wb\|^2\,. \nonumber
\end{align}
The values $\Omega>0$ and $A>0$ are hyperparameters of the algorithm.
Note that the constants $\omega$ and $a$ of Theorem~\ref{theo:pacbayesdabound_catoni_bis} can be recovered from any $\Omega$ and $A$.

Equation~\eqref{eq:probpbda2} is difficult to minimize by gradient descent, as it contains an absolute value and it is highly non-convex.
To make the optimization problem more tractable, we replace the loss function $\Phirisk$ by its convex relaxation $\Phic$ (as in Section~\ref{section:pbgd3_convex}).
Even if this optimization task is still not convex ($\Phidis$ is quasiconcave), our empirical study shows no need to perform many restarts while performing gradient descent to find a suitable solution.\footnote{We observe empirically that a good strategy is to first find the vector $\wb$ minimizing the convex problem of \PBGD described in Section~\ref{section:pbgd3_convex}, and then use this $\wb$ as a starting point for the gradient descent of \PBDA.}  
We name this domain adaptation algorithm \PBDA.

To sum up, given a source sample $S = \{(\xb^s_i, y^s_i)\}_{i=1}^m$, a target sample \mbox{$T = \{(\xb^t_i)\}_{i=1}^m$}, and hyperparameters $\Omega$ and $A$, the algorithm \PBDA performs gradient descent to minimize the following objective function:
\begin{equation} \label{eq:probpbda3}
G(\wb) = \Omega \sum_{i=1}^m  \Phic \left(  y^s_i \tfrac{\wb \cdot \xb^s_i}{\|\xb^s_i\|}  \right)   +  
A  \left| \sum_{i=1}^m \left[ \Phidis \left(  \tfrac{\wb \cdot \xb^s_i}{\|\xb^s_i\|}  \right)    -  \Phidis \left(  \tfrac{\wb \cdot \xb^t_i}{\|\xb^t_i\|}  \right)  \right]  \right|
  +   \tfrac{1}{2}\|\wb\|^2\,,
\end{equation}
where
$\Phic(x)  \eqdef  \max \big\{\Phirisk(x),\, \frac{1}{2} {-} \frac{x}{\sqrt{2\pii}} \big\}$
and
$\Phidis(x)\ \eqdef \ 2\,\Phirisk(x)\,\Phirisk(-x)$
have been defined by Equations~\eqref{eq:Phic} and~\eqref{eq:phidis}.
Figure~\ref{fig:phida_a} illustrates these three functions.

The gradient $\nabla G (\wb)$ of the Equation~\eqref{eq:probpbda3} is then given by
\begin{align*}
\nabla G (\wb)\, =\,&\Omega \sum_{i=1}^m 
\Phic' \LP\frac{y^s_i\wb\cdot\xb^s_i}{\|\xb^s_i\|} \RPP  \frac{y^s_i\xb^s_i}{\|\xb^s_i\|}\\
&+ s{\times}A\left(\sum_{i=1}^m \left[\Phidis' \LP\frac{\wb\cdot\xb^t_i}{\|\xb^t_i\|} \RPP  \frac{\xb^t_i}{\|\xb^t_i\|} -
\Phidis' \LP\frac{\wb\cdot\xb^s_i}{\|\xb^s_i\|} \RPP  \frac{\xb^s_i}{\|\xb^s_i\|} \right]\right) +\wb\,,
\end{align*}
where $\Phic'(x)$ and $\Phidis'(x)$ are respectively the derivatives of functions $\Phic$ and $\Phidis$ evaluated at point $x$, and 
$$s = \sgn \left(\ \displaystyle\sum_{i=1}^m  \left[ \Phidis \left(  \frac{\wb \cdot \xb^s_i}{\|\xb^s_i\|}  \right)    -  \Phidis \left(   \displaystyle\frac{\wb \cdot \xb^t_i}{\|\xb^t_i\|}  \right)  \right] \right). $$
We extend these equations to kernels in Section~\ref{section:kernelDA} below.

\subsubsection{Second Domain Adaptation Learning Algorithm ({\small DALC}).}
Now, Theorem~\ref{theo:catoni_new_general} specialized to linear classifiers gives the following.
	\begin{corollary}
		\label{cor:gen_bound linear2}
		For any domains $\PS$ and $\PT$ over $\X{\times}\Y$, any $\delta{\in}(0,1]$, any $b {>} 0$ and $c{>}0$, with a probability at least $1{-}\delta$ over the choices of $S{\sim}(\PS)^\ms$ and $ T {\sim} (\DT)^\mt$, we have, for all $\wb\in\Rbb$,
		\begin{align*}
		\RPT(h_\wb) & \leq\, c'\,\dT(\posterior_\wb)  +2\,b'\,\eS(\posterior_\wb)  + 2\,\ets  + 2\left(\tfrac{c'}{\mt\times c} + \tfrac{b'}{\ms\times b} \right)   \Big(\|\wb\|^2 + \ln \tfrac{2}{\delta}\Big)\,,
                \end{align*}
where $\dT(\Q_\wb)$ and $\eS(\posterior_\wb) $ are the empirical estimations of the target voters' disagreement and the source joint error, $b'=\frac{b}{1-e^{-b}}\,\binf$, and $c'=\frac{c}{1-e^{-c}}$.
	\end{corollary}
	For a source $S  {=} \{(\xbf^s_i, y^s_i)\}_{i=1}^{\ms}$ and a target $T  {=}  \{(\xbf^t_i)\}_{i=1}^{\mt}$ samples of potentially \emph{different size},  and some hyperparameters $B{>}0$, $C{>}0$, 	minimizing the next objective function {\it w.r.t} $\wb{\in}\Rbb$ is equivalent to minimize the above bound.
	\begin{align}
		\nonumber &C\,\dT(\posterior_\wb) + B\,\eS(\posterior_\wb) 	+  \|\wb\|^2 \\ 
		= \  &C \sum_{i=1}^{\mt} \Phidis \left(  \frac{\wb \cdot \xbf^t_i}{\|\xbf^t_i\|}  \right)  \!+ B\sum_{i=1}^\ms \Phierr \left(  y^s_i\, \frac{\wb \cdot \xbf^s_i}{\|\xbf^s_i\|}  \right) +\|\wb\|^2\,. \label{eq:ca}
	\end{align}
	We call the optimization of Equation~\eqref{eq:ca} by gradient descent the \algo algorithm, for Domain Adaptation of Linear Classifiers.
The gradient of Equation~\eqref{eq:ca} is 
	\begin{align*}
 C \sum_{i=1}^{\mt} \Phidis' \left(  \frac{\wb \cdot \xbf^t_i}{\|\xbf^t_i\|}  \right)    \frac{\xbf^t_i}{\|\xb^t_i\|}
	\!+ B\sum_{i=1}^\ms \Phierr' \left(  y^s_i\, \frac{\wb \cdot \xbf^s_i}{\|\xbf^s_i\|}  \right)   \frac{y^s_i\xb^s_i}{\|\xb^s_i\|}
	+\frac12 \,\wb\,.
	\end{align*}
Contrary to the algorithm \PBDA described above, our empirical study shows that there is no need to convexify any component of Equation~\eqref{eq:ca}: We obtain as good prediction accuracy when we initialize gradient descent to a uniform vector ($w_i=\frac{1}{d}$ for $i\in\{1,\ldots,d\}$) and when we perform multiple restarts from random initializations.
Even if the objective function is not convex, the gradient descent is easy to perform.
Indeed, $\Phidis$ is smooth and its derivative is continuous, in contrast with the absolute value of $\deswST$ in Equation~\eqref{eq:probpbda2} (see also the forthcoming toy experiment of Figure~\ref{fig:phida_d}). 
Thus, the actual optimization problem of \DALC is closer to the theoretical analysis than the \PBDA one.

\subsubsection{Using a Kernel Function}
\label{section:kernelDA}

Like the algorithm \PBGD in Subsection~\ref{section:pbgd_kernel}, the kernel trick applies to \PBDA and \algo.
Given a kernel $k\!:\!\R^d {\times} \R^d{\rightarrow}\R$, one can express a linear classifier in an \emph{RKHS} by a dual weight vector $\alphab\in \R^{\ms+\mt}$,
\begin{equation*}
h_\wb(\xb) \ =\ 
\sgn\left[
\sum_{i=1}^m \alpha_i k(\sbf^s_i, \xb) +  \sum_{i=1}^m \alpha_{i+m} k(\tbf^t_i, \xb)
\right].
\end{equation*}
Let
$S = \{(\xbf^s_i,y^s_i)\}_{i=1}^{\ms}$,\,\, $T = \{\xbf^t_i\}_{i=1}^{\mt}$\,and\, $\m=\ms+\mt$.
We denote $K$ the kernel matrix of size $\m\times\m$ such as $K_{i,j} \eqdef k(\xb_i, \xb_j)\,,$ where
$$\xb_\# \, =\, 
\begin{cases}
\xbf^s_i & \mbox{if } \# \leq \ms  \quad\mbox{ (source examples)} \\
\xbf^t_{\#-\ms} &\mbox{otherwise.} \quad\mbox{ (target examples)}
\end{cases} 
$$

On the one hand, in that case, with $\ms=\mt=m$, the objective function of \PBDA (Equation~\ref{eq:probpbda3}) is rewritten in terms of the vector $\ab = (\alpha_1,\alpha_2, \ldots\alpha_{2m})$ as
\begin{align*} 
G(\ab) &=  
 \Omega \sum_{i=1}^m  \Phic \left(  y_i \tfrac{\sum_{j=1}^{2m} \alpha_j K_{i,j}}{ \sqrt{K_{i,i}} }  \right) \\ 
 &{}+ 
A  \left| \sum_{i=1}^m  \left[\Phidis\! \left(  \tfrac{\sum_{j=1}^{2m} \alpha_j K_{i,j}}{ \sqrt{K_{i,i}} }  \right) \! -  \Phidis\! \left(  \tfrac{\sum_{j=1}^{2m} \alpha_j K_{i+m,j}}{ \sqrt{K_{i+m,i+m}} }  \right) \right] \right| 
+ \tfrac{1}{2} \sum_{i=1}^{2m} \sum_{j=1}^{2m} \alpha_i \alpha_j K_{i,j} \,.
\end{align*}

On the other hand,  the objective function of \DALC (Equation~\ref{eq:ca}) can be rewritten in terms of the vector 
$\alphab = (\alpha_1,\alpha_2, \ldots\alpha_{\m})$
 as
\begin{align*} 
C  \, \sum_{\mathclap{i=\ms+1}}^\m 
\Phidis\left(  \tfrac{\sum_{j=1}^{\m} \alpha_j K_{i,j}}{ \sqrt{K_{i,i}} }  \right)  
+B \sum_{i=1}^\ms 
 \Phierr\left(  y_i \tfrac{\sum_{j=1}^{\m} \alpha_j K_{i,j}}{ \sqrt{K_{i,i}} }  \right)
+\sum_{i=1}^{\m} \sum_{j=1}^{\m} \alpha_i \alpha_j K_{i,j} \,.
\end{align*}
To perform the gradient descent in terms of dual weights $\alphab$, we start the gradient descent from the point $\alpha_i = \frac{y_i}{\m}$ for $i\in\{1,\ldots,\ms\}$, and $\alpha_i = \frac1\m$ for $i\in\{\ms+1,\ldots,\m\}$.

\subsection{Illustration on a Toy Dataset}
\label{section:dalc_illustration}

\begin{figure}[h!]
\subfloat[Loss functions given by the specialization to linear classifiers.]{ \label{fig:phida_a} 
 \raisebox{-.5\height}{\includegraphics[width=0.4\textwidth]{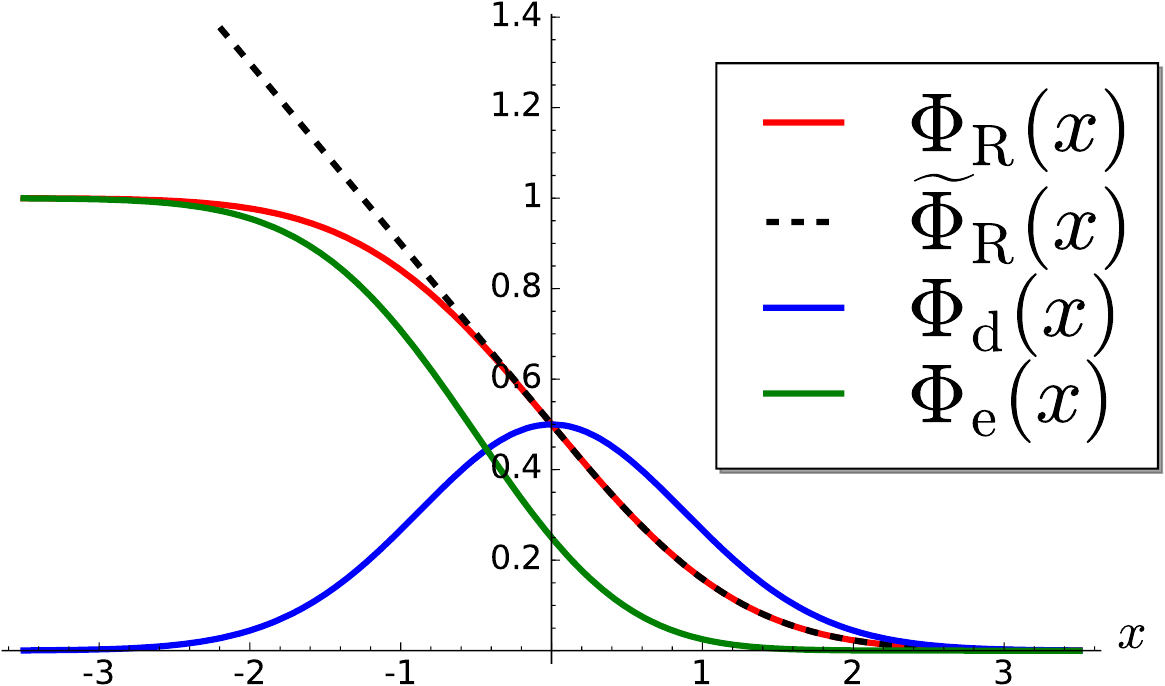}}
}
\subfloat[Loss functions definitions and their derivatives.]{ \label{fig:phida_b} \footnotesize
	\def\arraystretch{1.9}
\begin{tabular}{|l@{\,}|l@{\,}|}
\hline	
\hfill\sc function & \hfill\sc derivative \\\hline
 $\Phirisk(x)=
 \frac{1}{2}  \big[1 - \Erf\big(\frac{x}{\sqrt{2}}\big) \big]$ 
 & 
  $\Phirisk'(x) = {-}\frac{1}{\sqrt{2\pii}} e^{{-}\frac{1}{2} x^2}$ 
   \\\hline
   \begin{minipage}{3cm}$\Phic(x) =$\\
   \phantom.~~~$ \max \big\{\Phirisk(x), \frac{1}{2} {-} \frac{x}{\sqrt{2\pii}} \big\}$
   \end{minipage}
  &
  $ \Phirisk'\big(\max\{0,x\}\,\big)$
\\\hline
$\Phidis(x) = 2\,\Phirisk(x)\,\Phirisk(-x)$
&
  ${-}\sqrt{\frac{2}{\pii}} \Erf\big(\frac{x}{\sqrt{2}}\big)  e^{{-}\frac{1}{2} x^2}$ 
\\\hline
  $\Phierr(x)= \big[\Phirisk(x)\big]^2$
  &
 $2\,\Phirisk(x)\,\Phirisk'(x)$
 \\\hline
\end{tabular}
}\\
\subfloat[Toy dataset, and the decision boundary for $\theta=\frac{\pi}{4}$
(matching the vertical line of Fig.(d)). Source points are red and green, target points are black.]{ \label{fig:phida_c} 

\raisebox{-.5\height}{\includegraphics[width=0.4\textwidth]{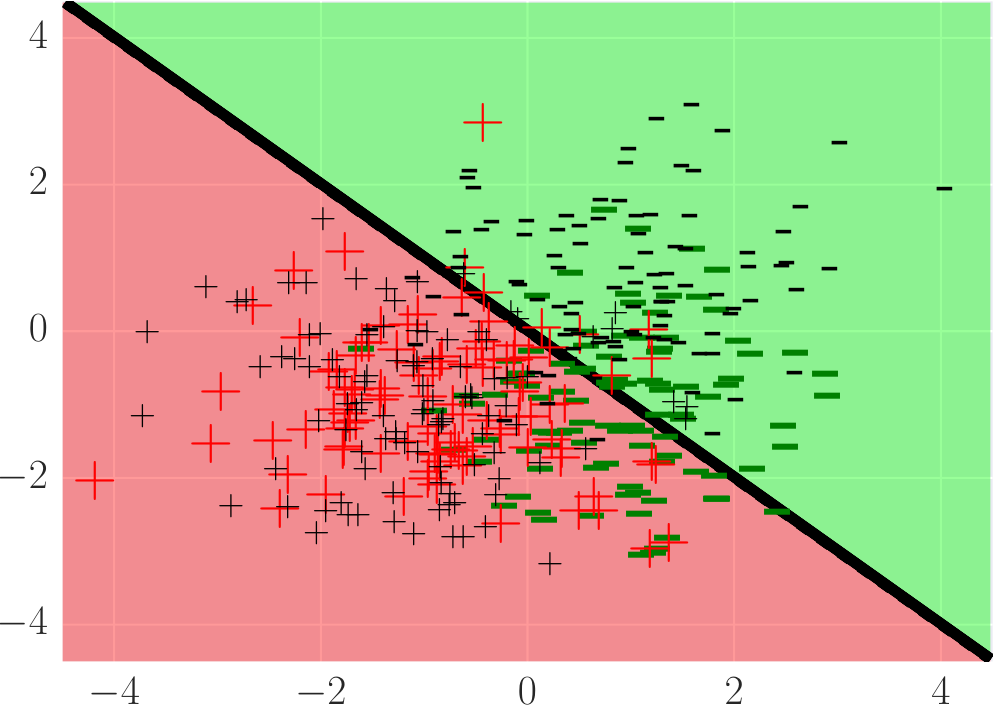}}
}\quad
\subfloat[\PBDA and \DALC loss values, according to $\theta$. \PBDA performs a trade-off between $\deswST$ and $\RS(G_{\posterior_\wb})$ (or its convex surrogate, see the red dashed line); \DALC's  trade-off is between $\dT(\Q_\wb)$ and $\eS(\posterior_\wb)$.]{ \label{fig:phida_d} 
\raisebox{-.5\height}{\includegraphics[width=0.55\textwidth]{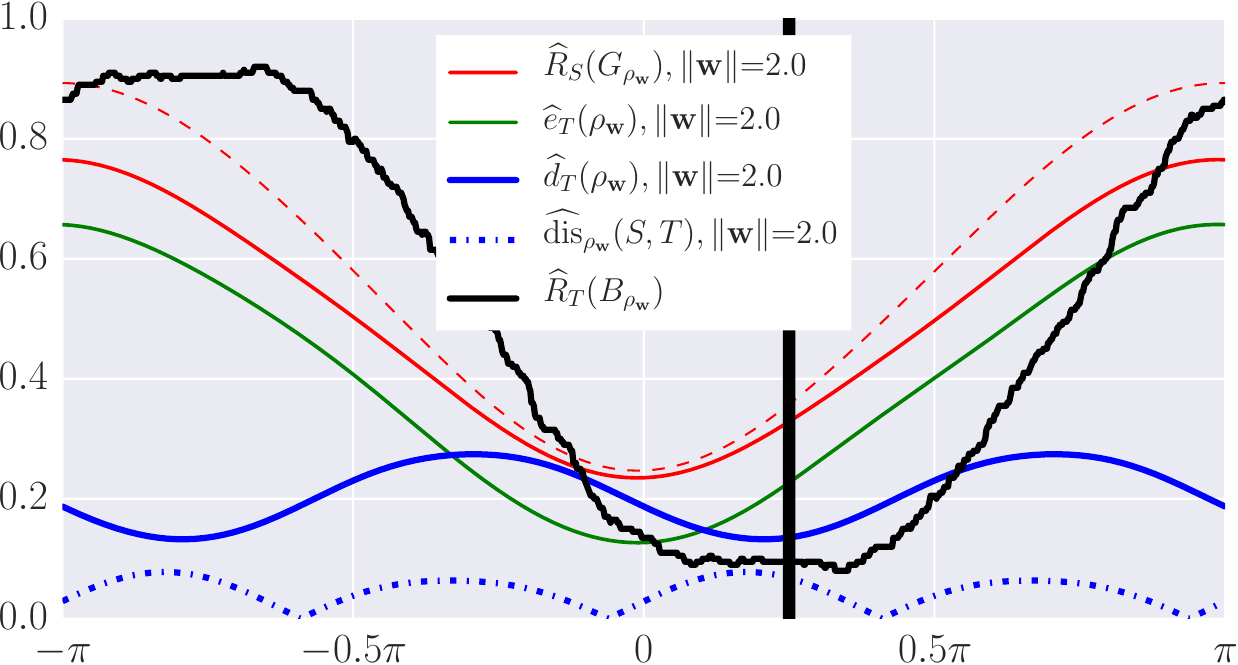}}
}
	\caption{Understanding \PBDA and \DALC domain adaptation learning algorithms in terms of loss functions. Upper Figures (a-b) show the loss functions, and lower Figures (c-d) illustrate the behavior on a toy dataset. }
\label{fig:phida} 
\end{figure}

To illustrate and compare the trade-offs of both algorithms \PBDA and \DALC, we extend the toy experiment of Subsection~\ref{section:pbgd_illustration} (see Figure~\ref{fig:phiclassic}). 
To obtain the two-dimensional dataset illustrated by
Figure~\ref{fig:phida_c}, we use, as the source sample, the $200$
examples of the supervised experiment---generated by Gaussians of mean
$(-1,-1)$ for the positives and $(-1,1)$ for the negatives (see Figure~\ref{fig:phiclassic_c}). 
Then, we generate $100$ positive target examples according to a Gaussian of mean $(-1,-1)$ and $100$ negative target examples according to a Gaussian of mean $(1,1)$.
All Gaussian distributions have unit variance. 
Note that positive source and target examples are generated by the same distribution. 

 We study linear classifiers $h_\wb$, with $\wb=2\,(\cos\theta, \sin\theta)\in\Rbb^2$.
That is, we fix the norm value $\|\wb\|=2$. 
Figure~\ref{fig:phida_d} shows the quantities varying in our two domain adaptation approaches while rotating the decision boundary $\theta\in[-\pi,\pi]$ around the origin. 
On the one hand, \PBDA algorithm minimizes a trade-off between the domain disagreement $\deswST$ and the source Gibbs risk $\RS(G_{\posterior_\wb})$ convex surrogate given by Equation~\eqref{eq:Phic}. 
On the other hand, \DALC minimizes a trade-off between the target disagreement $\dT(\Q_\wb)$ and source joint error $\eS(\posterior_\wb) $. 
From both Figures~\ref{fig:phida_a} and~\ref{fig:phida_d}, we see that the Gibbs risk, its convex surrogate, and the joint error behave similarly; they are following the linear classifier accuracy on the source sample.
However, the domains' divergence and the target joint error values notably differ for the experiment of Figure~\ref{fig:phida_d}\,:  When the target accuracy is optimal (\ie, $\theta\approx\frac{\pi}{4}$) the target disagreement is close to its lowest value, while it is the opposite for the domain divergence. 
Thus, provided that the hyperparameters handling the trade-off between $\dT(\Q_\wb)$ and $\eS(\posterior_\wb) $ are well chosen, the \DALC minimization procedure is able to find a solution \emph{close to} the one minimizing the target risk.  
On the contrary, for all hyperparameters, \PBDA will prefer the solution that minimizes the source risk ($\theta\approx0$), as it minimizes $\deswST$ and $\RS(G_{\posterior_\wb})$ simultaneously.

\section{Experiments}
\label{sec:expe}

Our domain adaptation algorithms \PBDA\footnote{\PBDA's code is available here: \url{https://github.com/pgermain/pbda}} and \DALC\footnote{\DALC's code is available here: \url{https://github.com/GRAAL-Research/domain_adaptation_of_linear_classifiers}} have been evaluated on a toy problem and a sentiment dataset. In both cases,
we minimize the objective function using a \emph{Broyden-Fletcher-Goldfarb-Shanno method (BFGS)} implemented in the \emph{scipy} python library.

\subsection{Toy Problem: Two Inter-Twinning Moons}

Figure~\ref{fig:moons} illustrates the behavior of the decision boundary of our algorithms \PBDA and \algo on an intertwining moons toy problem,\footnote{We generate each pair of moons with the {\tt make\_moons} function provided in {\tt scikit-learn}~\citep{scikit-learn}.} where each moon corresponds to a label. The target domain, for which we have no label, is a rotation of the source one. The figure shows clearly that \PBDA and \algo succeed to adapt to the target domain, even for a rotation angle of $50\degree$.
We see that our algorithms do not rely on the restrictive \emph{covariate shift} assumption, as some source examples are misclassified.
This behavior illustrates the \PBDA and \algo trade-off in action, that concede some errors on the source sample to lower the disagreement on the target sample.

\subsection{Sentiment Analysis Dataset}
\label{sec:sentiments}
We consider the popular {\it Amazon reviews} dataset \citep{BlitzerMP06} composed of reviews of four types of {\it Amazon.com}$^{\copyright}$~products (books, DVDs, electronics, kitchen appliances). Originally, the reviews are encoded in a bag-of-words representation (unigrams and bigrams) of approximately $100, 000$ features, and the labels are user rating between one and five stars. 
For sake of simplicity, we adopt the pre-processing proposed by \citet{ChenWB11}.
Hence, we tackle a binary classification task: a review is labeled $+1$ for a rank higher than $3$ stars, and $-1$ for a rank lower or equal to $3$ stars.
Also, the feature space dimensionality is reduced as follows: \citet{ChenWB11} only kept the features that appear at least ten times in a particular domain adaptation task (about $40, 000$ features remain), and pre-processed the data with a standard tf-idf re-weighting.
Considering each type of product as a domain, we perform twelve domain adaptation tasks. For instance, ``books$\rightarrow$DVDs'' corresponds to the task for which books is the source domain and DVDs the target one.
The learning algorithms are trained with $2, 000$ labeled source examples and $2, 000$ unlabeled target examples, and we evaluate them on the same separate target test sets proposed by \citet{ChenWB11} (between $3, 000$ and $6, 000$ examples).

\subsubsection{Experiment Details}

\PBDA and \DALC with a linear kernel have been compared with:  
\begin{itemize}
	\item \SVM learned only from the source domain without adaptation.
	We made use of the SVM-light library~\citep{Joachims99}.
	\item \PBGD, presented in Section~\ref{sec:pbgd}, and learned only from the source domain \mbox{without} adaptation.
	\item \DASVM of \citet{BruzzoneM10S}, an iterative domain adaptation algorithm which aims to maximize iteratively a notion of margin on self-labeled target examples. 
	We implemented DASVM with the LibSVM library \citep{libsvm}.
	\item \CODA of \citet{ChenWB11}, a co-training domain adaptation algorithm, which looks iteratively for target features related to the training set. We used the implementation provided by the authors. 
	Note that \citet{ChenWB11} have shown best results on the dataset considered in Section~\ref{sec:sentiments}.
\end{itemize}
Each parameter is selected with a grid search via a  classical \mbox{$5$-folds} cross-validation (${}^{CV}$) on the source sample for \PBGD and \SVM, and via a \mbox{$5$-folds} reverse/circular validation (${}^{RCV}$) on the source and the (unlabeled) target samples for \CODA, \DASVM, \PBDA, and \DALC.
We describe this latter method in the following subsection.
For \PBDA, respectively \DALC, we search on a \mbox{$20\times 20$} parameter grid for a $\Omega$, respectively $C$, between $0.01$ and $10^6$ and a parameter~$A$, respectively $B$, between $1.0$ and $10^8$, both on a logarithm scale.

Note that we did not compare the performance of our algorithms to state-osuf-the art (deep) representation learning techniques \citep{Chen12,ganin-16,DingF18,ShuBNE18,LiLHZS19,SebagHSSWA19}, but to methods that learn a predictor directly on the original input space. Both strategies are not to be opposed, as one can learn a common representation space for the source and the target domains, and afterwards learn a predictor that optimizes an adaptation criteria within that new space.

\subsubsection{A Note about the Reverse Validation}

\begin{figure}[t]
	\centering \includegraphics[width=0.75\textwidth]{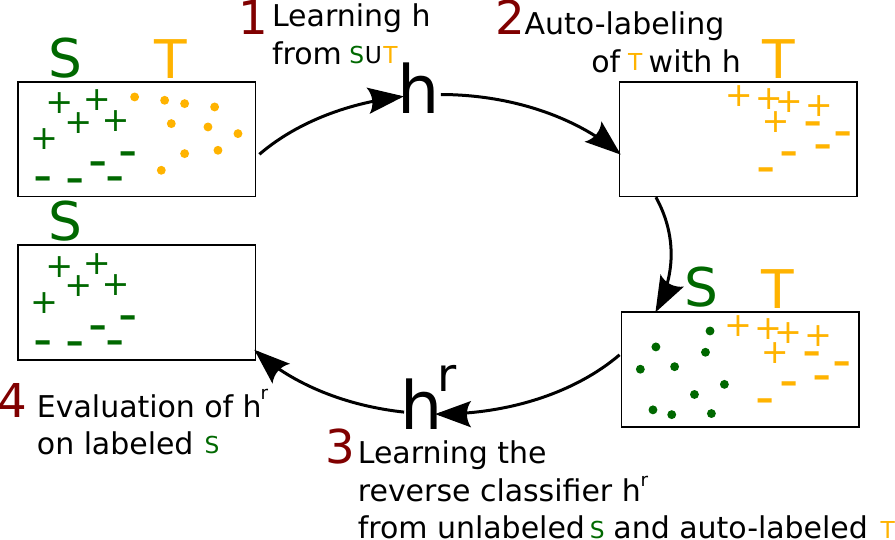}
	\caption{\label{fig:rev}The principle of the reverse/circular validation in our setting. 
	} 
\end{figure}

A crucial question in domain adaptation is the validation of the hyperparameters.
One solution is to follow the principle proposed by \citet{Zhong-ECML10} which relies on the use of a reverse validation approach.
This approach is based on a so-called reverse classifier evaluated on the source domain.
We propose to follow it for tuning the parameters of \DALC, \PBDA, \DASVM and \CODA.
Note that \citet{BruzzoneM10S} have proposed a similar method, called circular validation, in the context of \DASVM.

Concretely, in the current setting, given \mbox{$k$-folds} on the source labeled sample, $S=S_1\cup\ldots\cup S_k$, \mbox{$k$-folds} on the unlabeled target sample, $T=T_1\cup\ldots\cup T_k$) and a learning algorithm (parametrized by a fixed tuple of hyperparameters),  the reverse cross-validation risk on the $i^{\rm th}$ fold is computed as follows.
Firstly, the source set $S\!\setminus\! S_i$ is used as a labeled sample and the target set  $T\!\setminus\! T_i$ is used as an unlabeled sample for learning a classifier $h'$.
Secondly, using the same algorithm, a reverse classifier $h'^r$ is learned using the \emph{self-labeled} sample $\{(\xbf,  h'(\xbf))\}_{\xbf\in T\!\setminus\! T_i}$ as the source set and the unlabeled part of $S\!\setminus\! S_i$ as target sample.
Finally, the reverse classifier $h'^r$ is evaluated on $S_i$. 
We summarize this principle on Figure~\ref{fig:rev}. The process is repeated $k$ times to obtain the reverse cross-validation risk averaged across all folds.

\subsubsection{Empirical Results}
Table \ref{tab:res_sentiments} contains the test accuracies on the sentiment analysis dataset.
We make the following observations.
Above all,  the domain adaptation approaches provide the best average results, implying that tackling this problem with a domain adaptation method is reasonable. Then, our method \algo~based on the novel domain adaptation analysis is the best algorithm overall on this task. 
	Except for the two adaptive tasks between ``electronics'' and ``DVDs'', \algo~is either the best one (five times), or the second one (five times). 
	Moreover, according to a Wilcoxon signed rank test with a $5\%$ significance level, we obtain a probability of $89.5\%$ that \DALC is better than \PBDA. This test tends to confirm that the analysis with the new perspective improves the analysis based on a domains' divergence point of view. 
Moreover, \PBDA is on average better than \CODA, but less accurate than \DASVM. However, \PBDA is competitive: the results are not significantly different from \CODA and \DASVM. 
It is important to notice that \DALC and \PBDA are significantly faster than \CODA and \DASVM: These two algorithms are based on costly iterative procedures increasing the running time by at least a factor of five in comparison of \DALC and \PBDA. 
In fact, the clear advantage of the PAC-Bayesian approach is that we jointly optimize the terms of our bounds in one step.

\begin{figure}[t]
\centering
\includegraphics[width=0.3\textwidth,trim=20mm 10mm 13mm 14mm]{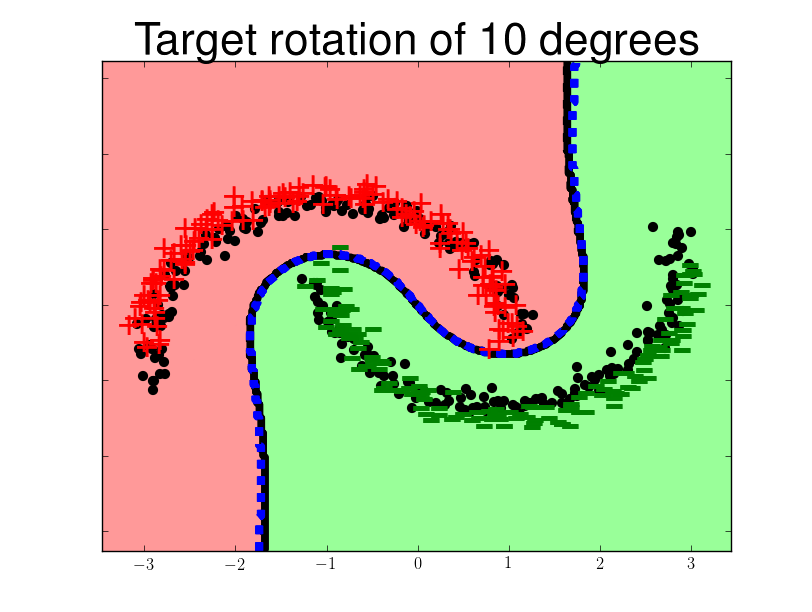}\hfill
\includegraphics[width=0.3\textwidth,trim=17mm 10mm 17mm 14mm]{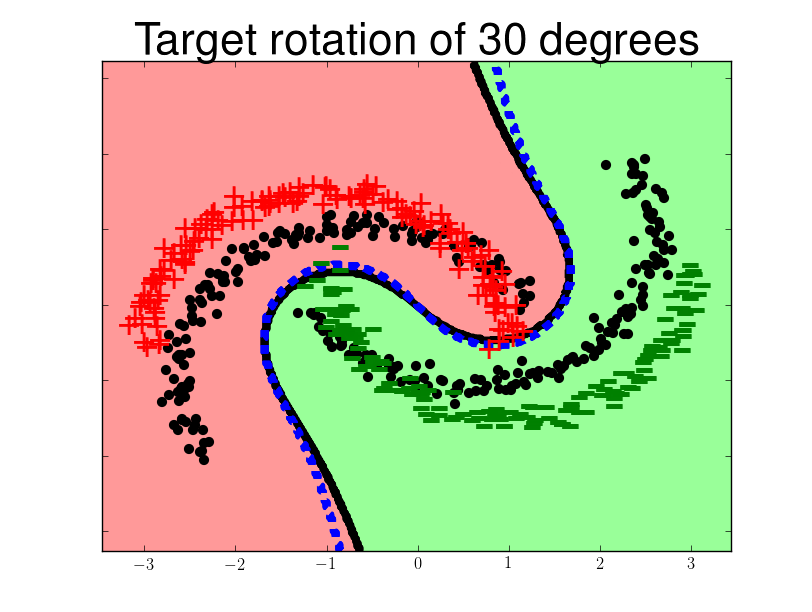}\hfill
\includegraphics[width=0.3\textwidth,trim=17mm 10mm 17mm 14mm]{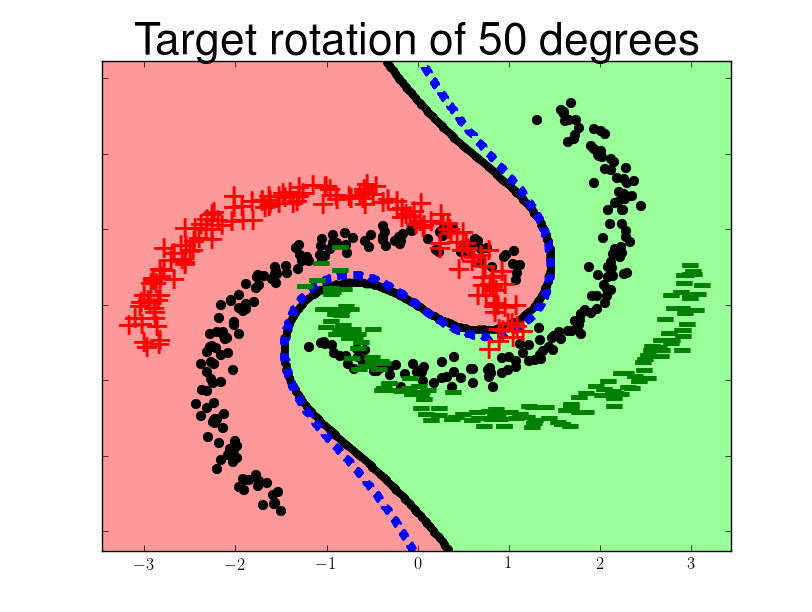}
\caption{Decision boundaries of \PBDA~(in blue dashed) and \algo~(in black) on the \emph{intertwining moons} toy problem, for fixed parameters $\alpha=A=1$ and $B=C=1$, and an RBF kernel $k(\xb,\xb')=\exp({- \|\xb-\xb'\|^2})$. The target points are black. The positive, respectively negative, source points are red, respectively green.
  \label{fig:moons}}
\end{figure}

\begin{table}[t]
\centering\small
\caption{Error rates for the sentiment analysis dataset. B, D, E, K respectively denotes  Books, DVDs,  Electronics,  Kitchen. In {\bf bold} are highlighted the best results, in {\it italic} the second ones. 
\label{tab:res_sentiments}}
\rowcolors{2}{}{black!10}
\begin{tabular}{|c||cc|cc|cc|}
\toprule
$ $ & \PBGD$^{\!CV}\!\!\! $ & \SVM$^{\!CV}\! $ & \DASVM$^{\!RCV}\!\!\! $ & \CODA$^{\!RCV}\! $ & \PBDA$^{\!RCV}\!\!\! $ & \DALC$^{\!RCV}\!$ \\ 
\midrule
B$\rightarrow$D & $ {\bf 0.174}$ & $0.179$ & $0.193$ & $0.181$ & $0.183$ &  ${\it 0.178}$ \\ 
B$\rightarrow$E & $0.275$ & $0.290$ & ${\it 0.226}$ & $0.232$ & $0.263$  & ${\bf 0.212}$\\ 
B$\rightarrow$K & $0.236$ & $0.251$ & ${\bf 0.179}$ & ${ 0.215}$ & $0.229$ & ${\it 0.194}$\\  
D$\rightarrow$B & ${\it 0.192}$ & $0.203$ & $0.202$ & $0.217$ & $0.197$ & ${\bf 0.186}$ \\ 
D$\rightarrow$E & $0.256$ & $0.269$ & ${\bf 0.186}$ & ${\it 0.214}$ & $0.241$ & $0.245$\\ 
D$\rightarrow$K & $0.211$ & $0.232$ & $0.183$ & ${\it 0.181}$ & $0.186$ & ${\bf 0.175}$\\ 
E$\rightarrow$B & $0.268$ & $0.287$ & $0.305$ & $0.275$ & ${\bf 0.232}$ & ${\it 0.240}$\\ 
E$\rightarrow$D & $0.245$ & $0.267$ & ${\bf 0.214}$ & $0.239$ & ${\it 0.221}$ & $0.256$\\ 
E$\rightarrow$K & ${\it 0.127}$ & $0.129$ & $0.149$ & $0.134$ & $0.141$ & ${\bf 0.123}$\\
K$\rightarrow$B & $0.255$ & $0.267$ & $ 0.259$ & ${\it 0.247}$ & ${\it 0.247}$ & ${\bf 0.236}$ \\ 
K$\rightarrow$D & $0.244$ & $0.253$ & ${\bf 0.198}$ & $0.238$ & ${0.233}$ & ${\it 0.225}$ \\ 
K$\rightarrow$E & $0.235$ & $0.149$ & $0.157 $ & $0.153$ & ${\bf 0.129}$ & ${\it 0.131}$\\ 
\midrule
Average & $0.226 $ & $0.231 $ & ${\it 0.204} $ & $0.210 $ & ${0.208} $ & ${\bf 0.200}$\\ 
\bottomrule
\end{tabular}
\end{table}

\section{Conclusion and Future Work}
\label{sec:conclu}

In this paper, we present two domain adaptation analyses for the PAC-Bayesian framework that focuses on models that take the form of a majority vote over a set of classifiers: The first one is based on a common principle in domain adaptation, and the second one brings a novel perspective on domain adaptation.

To begin, we follow the underlying philosophy of the seminal works of \citet{BenDavid-NIPS06,BenDavid-MLJ2010} and \citet{Mansour-COLT09}; in other words, we derive an upper bound on the target risk (of the Gibbs classifier) thanks to a domains' divergence measure suitable for the PAC-Bayesian setting. 
We define this divergence as the average deviation between the disagreement over a set of classifiers on the source and target domains. 
This leads to a bound that takes the form of a trade-off between the source risk, the domains' divergence and a term that captures the ability to adapt for the current task.
Then, we propose another domain adaptation bound while taking advantage of the inherent behavior of the target risk in the PAC-Bayesian setting. 
We obtain a different upper bound that is expressed as a trade-off between the disagreement only on the target domain, the joint errors of the classifiers only on the source domain, and a term reflecting the worst-case error in regions where the source domain is non-informative.
To the best of our knowledge, a crucial novelty of this contribution is that the trade-off is controlled by a domains' divergence: Contrary to our first bound, the divergence is not an  additive term (as in many domain adaptation bounds) but is a factor weighing the importance of the source information. 

Our analyses, combined with PAC-Bayesian generalization bounds, lead to two new domain adaptation algorithms for linear classifiers: \PBDA associated with the previous works philosophy, and \DALC associated with the novel perspective.
The empirical experiments show that the two algorithms are competitive with other approaches, and that \DALC outperforms significantly \PBDA.
We believe that our PAC-Bayesian analyses open the door to develop new domain adaptation methods by making use of the possibilities offered by the PAC-Bayesian theory, and give rise to new interesting directions of research, among which the following ones.

Firstly, the PAC-Bayesian approach allows one to deal with an {\it a priori} belief about the classifiers accuracy; in this paper we opted for a non-informative prior that is a Gaussian centered at the origin of the linear classifier space. 
The question of finding a relevant prior in a domain adaptation situation is an exciting direction which could also be exploited when some few target labels are available.
Moreover, this notion of prior distribution could model information learned from previous tasks as pointed out by~\citet{pentina14}, or from other views/representations of the data~\cite{goyal2017pac}.
This suggests that we can extend our analyses to multisource domain adaptation \citep{Crammer2007learning,mansour2009domain,BenDavid-MLJ2010,HoffmanMohri2018} and lifelong learning where the objective is to perform well on future tasks, for which no data has been observed so far~\citep{ThrunM95}.

Another promising issue is to address the problem of the hyperparameter selection. 
Indeed, the adaptation capability of our algorithms \PBDA and \DALC could be even put further with a specific PAC-Bayesian validation procedure. 
An idea would be to propose a (reverse) validation technique that takes into account some particular prior distributions. 
Another solution could be to explicitly control the neglected terms in the domain adaptation bound.
This is also linked with model selection for domain adaptation tasks.

Besides, deriving a result similar to Equation~\eqref{eq:C-bound} (the $C$-bound) for domain adaptation could be of high interest. 
Indeed, such an approach considers the first two moments of the margin of the weighted majority vote. This could help to take into account both margin information over unlabeled data and the distribution disagreement (these two elements seem of crucial importance in domain adaptation).

Concerning \DALC, we would like to investigate the case where the domains' divergence can be estimated, {\it i.e.}, when the covariate shift assumption holds or when some target labels are available.
In this scenario, the domains' divergence might not be considered as a hyperparameter to tune.

The results obtained in this paper are dedicated to binary classification, one another interesting issue to tackle could be the case of multiclass or multilabel classification.

Last but not least, the non-estimable term of  \DALC---suggesting that the domains should live in the same regions---can be dealt with representation learning approach.
This could be an incentive to combine \DALC with existing representation learning techniques.

\section*{Acknowledgements}
This work was supported in part by the French projects LIVES {\scriptsize ANR-15-CE23-0026-03}, and in part by NSERC discovery grant {\scriptsize 262067}, and by the European Research Council under the European Unions Seventh Framework Programme (FP7/2007-2013)/ERC grant agreement no 308036.
 Computations were performed on Compute Canada and Calcul Qu\'ebec infrastructures (founded by CFI, NSERC and FRQ). We thank Christoph Lampert and Anastasia Pentina for helpful discussions. A part of the work of this paper was carried out while E. Morvant was affiliated with IST Austria, and while P. Germain was affiliated with D{\'e}partement d'informatique et de g\'enie logiciel, Universit{\'e} Laval, Qu{\'e}bec, Canada.

\appendix
\section{Some Tools} \label{section:tools}

\begin{lemma}[H\"older's Inequality]
\label{theo:holder}
Let $S$ be a measure space and let $(p, q)\in[1,\infty]^2$ with $\tfrac1p + \tfrac1q = 1$. Then, for all measurable real-valued functions $f$ and $g$ on $S$,
$$
 \|fg\|_{1}\leq \|f\|_{p}\|g\|_{q}.$$
\end{lemma}

\begin{lemma}[Markov's inequality]
\label{theo:markov}
Let $Z$ be a random variable and $t\geq 0$,  then 
$$
\Pr{(|Z|\geq t)} \ \leq \ \frac{\displaystyle \Esp(|Z|)}{t}\,.
$$
\end{lemma}

\begin{lemma}[Jensen's inequality]
\label{theo:jensen}
Let $Z$ be an integrable real-valued random variable and $g$ any function. 
If $g$ is convex, then
$$
\quad g(\Esp[Z])\ \leq\ \Esp[g(Z)]\,.
$$
\end{lemma}

\begin{lemma}[from Lemma 3 of \cite{Maurer04}]
\label{lem:maurer2}
Let $X {=} ( X_1,\dots,X_m)$ be a vector of {\it i.i.d.} random variables, $0\leq X_i\leq 1$, with $\Esp X_i = \mu$.  
Denote $X' {=} ( X'_1,\dots,X'_m)$, where~$X_i'$ is the unique Bernoulli ($\{0,1\}$-valued) random variable with $\Esp X_i' = \mu$. If $f:[0,1]^m\rightarrow \mathbb{R}$ is convex, then
$$
\Esp [f(X)] \ \leq \ \Esp [f(X')]\,.
$$
\end{lemma}

\begin{lemma}[from Inequalities 1 and 2 of \cite{Maurer04}]
\label{lem:maurer}
Let
 $X = ( X_1,\dots,X_m)$ be a vector of {\it i.i.d.} random variables, $0\leq X_i\leq 1$. 
 Then 
$$
\sqrt{m}\ \leq\ \Esp \exp \left[m\, \kl\left(\frac{1}{m}\sum_{i=1}^m X_i\,\Big\|\,\Esp[X_i]\right)\right]\ \leq\ 2\sqrt{m}\,.
$$ 
\end{lemma}

\begin{lemma} {\bf(Change of measure inequality\footnote{See \cite[Lemma 4]{seldin-10}, \cite[Equation 20]{mcallester-13}, or  \cite[Lemma 17]{graal-neverending}.})}
	 \label{lem:change-measure}
\,For any set $\Hcal$, for any distributions $\prior$ and $\posterior$ on $\Hcal$, and for any measurable function $\phi:\Hcal \to \mathbb{R}$,  we have
\begin{equation*}
\Esp_{f\sim \posterior} \phi(f) \ \leq \ \KL(\posterior\|\prior) + \ln \left( \Esp_{f\sim \prior} e^{\phi(f)} \right) \,.
\end{equation*}
\end{lemma}

\begin{lemma}[from Theorem 25 of \cite{graal-neverending}] \label{lem:2KL}
Given any set $\Hcal$, and any distributions $\prior$ and $\posterior$
on $\Hcal$, let  $\hat{\posterior}$ and $\hat{\prior}$ two distributions over $\Hcal^2$ such that $\hat{\posterior}(h,h') \eqdef \posterior(h)\posterior(h')$ and $\hat{\prior}(h,h') \eqdef \prior(h)\prior(h')$. Then
\begin{equation*}
\KL(\hat{\posterior}\|\hat{\prior})  \ =\   2\,\KL(\posterior\|\prior)\,.
\end{equation*}
\end{lemma}

\section{Proof of Equation~(10)}
\label{appendix:RSGw}
Given $(\xbf,y)\in\Rbb^d\times\{-1,1\}$ and $\wb\in\Rbb^d$, we consider---without loss of generality---a vector basis where  $y\frac{\xbf}{\|\xbf\|}$ is the first coordinate. Thus, the first component of any vector $\wb'\in\Rbb^d$ is given by
$w_1' = y\frac{\wbf'\! \cdot \xbf}{\|\xbf\|}$, which leads to
\begin{align*}
\Esp_{h_{\wbf'}\sim\posterior_\wb} \zoloss\big( h_{\wbf'}(\xb), y \big) 
& =  \Esp_{h_{\wbf'}\sim\posterior_\wb} \I\,\big[ h_{\wbf'}(\xb) \neq  y \big] \\    
&=   \Esp_{h_{\wbf'}\sim\posterior_\wb}    \I\,\big[y \,\wb' \!\cdot \xb \leq 0\big]\\ 
&=   \int_{ \mathbb{R}^d}   \frac{1}{\sqrt{(2\pii)^d}}
\exp\left({-\tfrac{1}{2}\|{\wbf'}-\wb\|^2}\right) \, \I\,\big[y \,\wb'\! \cdot \xb \leq 0\big]\, d\, \wb' \\ 
&=  \int_{ \mathbb{R}}   \frac{1}{\sqrt{2\pii}}
\exp\left({-\tfrac{1}{2}(w_1' - w_1)^2}\right) \, \I\,\big[w_1'\leq 0\big]\, d w_1' \\ 
&=  \int_{-\infty}^\infty   \frac{1}{\sqrt{2\pii}}
\exp\left({-\tfrac{1}{2}\,t^2}\right) \, \I\,\big[t\leq -w_1\big]\, d t \\ 
&=  \int_{-\infty}^{-w_1}  \frac{1}{\sqrt{2\pii}}
\exp\left({-\tfrac{1}{2}\,t^2}\right) \, d t \\ 
&=   1 - \Pr_{t\sim\mathcal{N}(0,1)}\!\left(t\ \leq\  y\, \tfrac{\wb \cdot \xb}{\|\xb\|} \right)\\
&=  \Phirisk\left(  y\, \tfrac{\wb \cdot \xb}{\|\xb\|}  \right),
\end{align*}
where we used $y \,\wb'\! \cdot \xb = w_1'\|\xbf\|$,\, $t\eqdots w_1'{-}w_1$\,, $w_1 = y\frac{\wbf \cdot \xbf}{\|\xbf\|}$\,, and the definition of $\Phirisk$ given by Equation~\eqref{eq:eq:gibbs_risk_linear}.
We then have
\begin{align*}
\RP(G_{\posterior_\wb}) \  
 = \Esp_{(\xb,y)\sim P_S} \Esp_{h_{\wbf'}\sim\posterior_\wb} \zoloss\big( h_{\wbf'}(\xb), y \big) 
 \ =  \Esp_{(\xb,y)\sim P_S} \Phirisk\left(  y\, \tfrac{\wb \cdot \xb}{\|\xb\|}  \right).
\end{align*}

\section{Proof of Theorem~6} 
\label{proof:catoni_general}

\newcommand{\LD}{\mathcal{L}_{\P}}
\newcommand{\LS}{\mathcal{L}_{S}}
\newcommand{\LSp}{\mathcal{L}_{S'}}

\begin{proof}
We use the shorthand notation\\
\centerline{
$\displaystyle\LD(h) \,=\, \Esp_{\mathclap{\exbf\sim\P}}  \ \  \ell(h, \xb,y)$
\mbox{ and }
$\displaystyle\LS(h) = \tfrac1m\sum_{\mathclap{\exbf\in S}}  \ell(h, \xb,y)\,.
$}

Consider any convex function $\Delta:[0,1]{\times}[0,1]\to\R$.
Applying consecutively Jensen's Inequality (Lemma~\ref{theo:jensen}) and the \emph{change of measure inequality} (Lemma~\ref{lem:change-measure}), we obtain
\begin{align*}
\forall \posterior \mbox{ on } \Hcal\,, \quad   m{\times}\Delta\left(
\Esp_{h\sim\posterior} \LS(h),  \Esp_{h\sim\posterior} \LD(h)
\right) 
&\ \leq \ 
\Esp_{h\sim\posterior}
m {\times} \Delta\left(
\LS(h), \LD(h)
\right)\\
&\ \leq \ \KL(\posterior\|\prior)
+ \ln \Big[
X_\prior(S)
\Big]\,,
\end{align*}
with
\begin{equation*}
X_\prior(S) = 
\Esp_{h\sim\prior}
e^{m {\times} \Delta\left(
 \LS(h),\, \LD(h)
\right)}.
\end{equation*}
Then, Markov's Inequality (Lemma~\ref{theo:markov}) gives
\begin{equation*}
\Pr_{S\sim\P^m} \left(
X_\prior (S) \leq \tfrac1\delta \Esp_{S'\sim\P^m} X_\prior (S')
\right)
\ \geq \ 
1{-}\delta\,, 
\end{equation*}
and 
\begin{align}
\nonumber
\Esp_{\mathclap{S'\sim\P^m}} \ X_\prior (S')
\ =&\ 
\Esp_{S'\sim\P^m} \Esp_{h\sim\prior}
e^{m {\times} \Delta\left(
 \LSp(h),\, \LD(h)
\right)}
 \\
\nonumber
 =& \ 
\Esp_{h\sim\prior} \Esp_{S'\sim\P^m} 
e^{m {\times} \Delta\left(
 \LSp(h),\, \LD(h)
\right)}
\\
\label{eq:aaa}
\leq & \ \Esp_{h\sim\prior} 
\sum_{k=0}^m
\binom{k}m (\LD(h))^k (1{-}\LD(h))^{m-k}
e^{m {\times} \Delta\left(
 \frac{k}m,\, \LD(h)
\right)},
\end{align}
where the last inequality is given by Lemma~\ref{lem:maurer}
(we have equality when the output of $\ell$ is in $\{0,1\}$).
As shown in \citet[Corollary 2.2]{germain2009pac}, by fixing
 $$\Delta(q,p) = -\alpha{\times} q -\ln[1{-}p\,(1{-}e^{-\alpha})]\,,$$
 Line~\eqref{eq:aaa} becomes equal to $1$, and then $\Esp_{S'\sim\P^m} X_\prior (S')\leq 1$.
Hence, with probability $1{-}\delta$ over the choice of $S\in \P^m$, we have
\begin{align*}
\forall \posterior \mbox{ on } \Hcal,  \
-\alpha \Esp_{h\sim\posterior} \LS(h) -\ln[1{-}\Esp_{h\sim\posterior} \LD(h)\,(1{-}e^{-\alpha})] 
 \leq  
 \frac{\KL(\posterior\|\prior)
 + \ln  \tfrac1\delta}{m}\,.
\end{align*}
By reorganizing the terms, we have, with probability $1{-}\delta$ over the choice of $S\in \P^m$, 
\begin{align*}
&\forall \posterior \mbox{ on } \Hcal, \
 \Esp_{h\sim\posterior} \LD(h) \leq  
\frac1{1{-}e^{-\alpha}}\left[1-
\exp\left(-\alpha \Esp_{h\sim\posterior} \LS(h) -
 \frac{\KL(\posterior\|\prior)
 + \ln  \tfrac1\delta}{m} \right)\right].
\end{align*}
The final result is obtained by using the inequality $1{-}\exp(-z) \leq z$. 
\end{proof}

\section*{References}
\bibliography{biblio}

\end{document}